%% file: main.tex
\documentclass[10pt]{article}

\usepackage[accepted]{rlj}   

\usepackage{microtype}
\usepackage{graphicx}
\usepackage{subcaption}
\usepackage{booktabs}

\usepackage{amssymb,mathtools,amsthm}
\usepackage{thm-restate}

\usepackage{xspace}

\usepackage[on]{editing}
\usepackage{multirow}

\usepackage{algorithm,algorithmic}
\renewcommand{\algorithmicrequire}{\textbf{Input:}}
\renewcommand{\algorithmicensure}{\textbf{Output:}}
\newcommand{\RETURN}{\textbf{return} }

\usepackage{silence}
\usepackage{tikz}
\usetikzlibrary{arrows.meta,positioning,decorations.pathreplacing,decorations.pathmorphing,patterns,calc}

\definecolor{betterred}{RGB}{228,26,28}
\definecolor{betterblue}{RGB}{55,126,184}
\definecolor{bettergreen}{RGB}{77,175,74}
\definecolor{betterpurple}{RGB}{152,78,163}
\definecolor{LightCyan}{rgb}{0.88,1,1}

\theoremstyle{plain}
\newtheorem{theorem}{Theorem}[section]
\newtheorem{proposition}[theorem]{Proposition}
\newtheorem{lemma}[theorem]{Lemma}
\newtheorem{corollary}[theorem]{Corollary}
\theoremstyle{definition}
\newtheorem{definition}[theorem]{Definition}
\newtheorem{assumption}[theorem]{Assumption}
\theoremstyle{remark}
\newtheorem{remark}[theorem]{Remark}

\newcommand{\E}{\mathbb{E}}
\newcommand{\Var}{\mathrm{Var}}
\newcommand{\Cov}{\mathrm{Cov}}
\newcommand{\Corr}{\mathrm{Corr}}
\newcommand{\R}{\mathbb{R}}

\newcommand{\U}{\mathbf{U}}
\newcommand{\V}{\mathbf{V}}
\newcommand{\X}{\mathbf{X}}  
\newcommand{\cM}{\mathcal{M}}

\newcommand{\cS}{\mathcal{S}}
\newcommand{\cA}{\mathcal{A}}

\title{Counterfactual Shapley Credit Assignment}
\setrunningtitle{Counterfactual Shapley Credit Assignment}

\author{Mingxuan Li\textsuperscript{1,$\dagger$}, Kai-Zhan Lee\textsuperscript{1,$\dagger$}, Elias Bareinboim\textsuperscript{1}}
\emails{ml@cs.columbia.edu, \ kl2792@columbia.edu, \ eb@cs.columbia.edu}
\affiliations{
$^{1}$\textbf{Department of Computer Science, Columbia University, New York, NY, USA}
\par
$^\dagger$ Equal contribution
}

\keywords{temporal credit assignment, counterfactual Shapley values, causal inference}

\summary{The Credit Assignment Problem (CAP) is fundamental to developing efficient and explainable Reinforcement Learning (RL) agents. Existing frameworks, whether relying on temporal contiguity or hindsight-conditioned reward reweighting, frequently fail to attribute properly between an agent's policy (skill) and environmental stochasticity (luck). A principled approach to CAP must isolate the true causal drivers of observed outcomes from spurious correlations and environmental randomness. We introduce Counterfactual Shapley Credit Assignment, a novel framework grounded in causal theory that attributes credit and blame via the Counterfactual Shapley Value ($\phi$-value). By redistributing environmental rewards, $\phi$-values enhance temporal credit assignment across three critical dimensions: sparse causality, high stochasticity, and delayed rewards, all while preserving the optimal policy. We derive a consistent estimator that computes $\phi$-values efficiently, enabling a new class of policy gradient methods, $\phi$-PPO, combined with Prioritized Trajectory Replay (PTR). Empirical results demonstrate that $\phi$-values align precisely to the ground truth causes of task rewards with superior sample efficiency in challenging environments where prior state-of-the-art methods fail to converge.}

\contribution{
We apply the Counterfactual Shapley Value ($\phi$-value) framework~\citet{lee2025} to RL temporal credit assignment, showing that $\phi$-values yield a reward redistribution that preserves the optimal policy and concentrates credit on causal actions. We show that the proposed framework satisfies the four causal credit assignment criteria: (1) Causal Admissibility (2) Causal Power (3) Causal Normality and (4) Causal Effect Scaling. The four criteria together guarantee that the proposed framework allocates credits faithful to how much causal contribution each action has to the outcome compared to a baseline policy, effectively rectifying the long-standing challenges in RL temporal credit assignment: sparsity, high stochasticity and delayed rewards.
}{
Prior credit assignment methods, including HCA~\citep{hindsightcredit}, CCA~\citep{pmlr-v139-mesnard21a}, CoCoA~\citep{meulemans2023would}, RUDDER~\citep{arjona2019rudder}, and Synthetic Returns~\citep{raposo2021synthetic}, are not grounded in counterfactuals in the Structural Causal Model (SCM) framework~\citep{pearl2009causality,bareinboim2022pearl}.
}

\contribution{
We derive a consistent $\phi$-value function estimator converging to the true causal contributions in probability given sufficient samples with amortized $O(1)$ time per action and a tunable bias--variance tradeoff via $\lambda$-bootstrapping.
}{
Previous work on estimating Shapley values~\citep{shapley:book1952} and $\phi$-values~\citep{lee2025} don't consider bootstrapping with value function approximators. Thus, they usually suffer from either exponential sample complexity or biased estimation. 
}

\contribution{
We introduce $\phi$-PPO with Prioritized Trajectory Replay (PTR) and demonstrate superior sample efficiency on the SkillLuck and Combinatorial Lock benchmarks, where prior methods either take much longer to learn or fail to converge; through ablation, we show that both $\phi$-values and PTR are necessary.
}{
Compared against PPO and REINFORCE~\citep{schulman2017ppo,williams1992simple}, HCA~\citep{harutyunyan2019hindsight}, CCA~\citep{pmlr-v139-mesnard21a}, CoCoA~\citep{meulemans2023would}, QCA~\citep{mesnard2023quantile}, RUDDER~\citep{arjona2019rudder}, Synthetic Returns~\citep{raposo2021synthetic}, and H-DICE~\citep{velu2024hindsightdice} in three classical tabular environments characterizing the challenges in temporal credit assignment.
}

\begin{document}

\makeCover
\maketitle

\input{sections/0-abstract}

\input{sections/1-intro}               
\input{sections/2-prelims}             
\input{sections/3-credit-assignment}   
\input{sections/4-estimating}          
\input{sections/5-experiments}         
\input{sections/6-conclusion}          

\section*{Acknowledgments}
This research is supported in part by the NSF, ONR, AFOSR, DoE, Amazon, JP Morgan, and The Alfred P. Sloan Foundation.

\bibliography{refs,bibwithID,extra}
\bibliographystyle{rlj}

\beginSupplementaryMaterials
\appendix

\vspace{0.3cm}
\noindent\textbf{Appendix Contents}
\begin{itemize}[leftmargin=20pt, topsep=4pt, itemsep=1pt, parsep=0pt]
    \item[{\ref{sec:full code}}] \hyperref[sec:full code]{Counterfactual Simulation}
    \item[{\ref{sec:phi-ppo}}] \hyperref[sec:phi-ppo]{$\phi$-PPO}
    \item[{\ref{sec:theory detail}}] \hyperref[sec:theory detail]{Theory Details}
    \item[{\ref{app:proofs}}] \hyperref[app:proofs]{Proofs}
    \item[{\ref{app:experiments}}] \hyperref[app:experiments]{Additional Experiments}
    \item[{\ref{app:limitations}}] \hyperref[app:limitations]{Limitations and Future Work}
    \item[{\ref{app:ca-methods}}] \hyperref[app:ca-methods]{Credit Assignment Methods: Detailed Comparison}
\end{itemize}
\vspace{0.3cm}

\input{sections/E-full-pseudocode}
\input{sections/D-phiPPO}
\input{sections/C-theorydetails}
\input{sections/A-proofs}
\input{sections/B-limitations}
\input{sections/B-related-methods}

\end{document}

%% file: sections/0-abstract.tex
\begin{abstract}
The Credit Assignment Problem (CAP) is fundamental to developing efficient and explainable 
Reinforcement Learning (RL) agents.
Existing frameworks, whether relying on temporal contiguity or hindsight-conditioned reweighting of rewards, frequently fail to distribute credit properly 
between an agent's policy (skill) and environmental stochasticity (luck). A principled approach to the CAP must isolate the causes of observed rewards from spurious correlated features and environmental randomness.
We introduce Counterfactual Shapley Credit Assignment, a novel causal credit assignment framework based on the counterfactual Shapley value ($\phi$-value). By redistributing rewards, $\phi$-values enhance credit assignment across three critical dimensions: high stochasticity, sparse causality, and delayed rewards, all while theoretically preserving the original optimal policy.
We derive a consistent estimator that computes each $\phi$-value in amortized constant time complexity, enabling a new class of policy gradient methods, $\phi$-PPO. Empirical results demonstrate that $\phi$-values align precisely 
with the ground truth causes of task rewards. Furthermore, we show its superior sample efficiency in challenging environments where prior state-of-the-art methods fail to converge. 
\end{abstract}

%% file: sections/1-intro.tex
\section{Introduction}
\label{sec:intro}

Reinforcement learning (RL) agents have achieved remarkable success in domains such as healthcare~\citep{yu2020reinforcementlearninghealthcaresurvey}, autonomous driving~\citep{shaoSurveyDeepReinforcement2019a}, and code generation~\citep{claude2024,yang2025qwen3technicalreport}, driven by advances in deep RL~\citep{mnihHumanlevelControlDeep2015a,silverMasteringGameGo2017,akkaya2019solvingcube,schrittwieserMasteringAtariGo2020,Guo_2025deepseek}.
As RL scales to more challenging domains with longer horizons, sparse and delayed rewards, and greater stochasticity, the long-standing \emph{Credit Assignment Problem} (CAP) grows increasingly relevant: learning depends on successfully ``distributing credit for success of a complex strategy among the many decisions that were involved'' \citep{minsky1961}.
We argue that this credit distribution must be causal: the credit each action receives should reflect its causal effect on the total discounted reward.

\begin{figure*}[t]
\vspace{-0.2in}
    \begin{minipage}[c]{.36\textwidth}
    \centering
        \input{figures/intro_credit_assignment_example}
    \end{minipage}
    \hfill
    \begin{subfigure}[c]{0.31\linewidth}
        \centering
        \includegraphics[width=\linewidth]{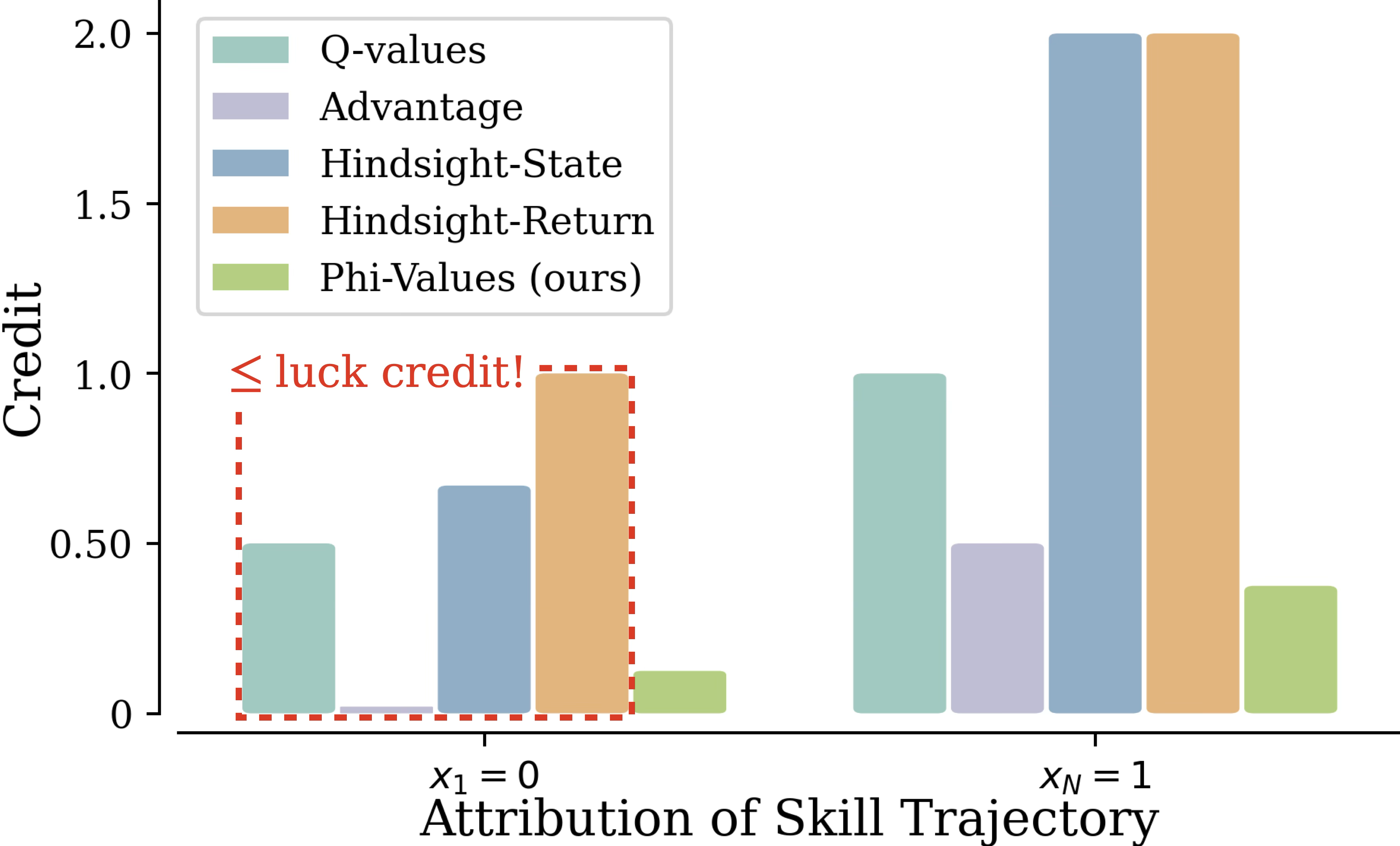}
    \end{subfigure}
    \hfill
    \begin{subfigure}[c]{0.31\linewidth}
        \centering
        \includegraphics[width=\linewidth]{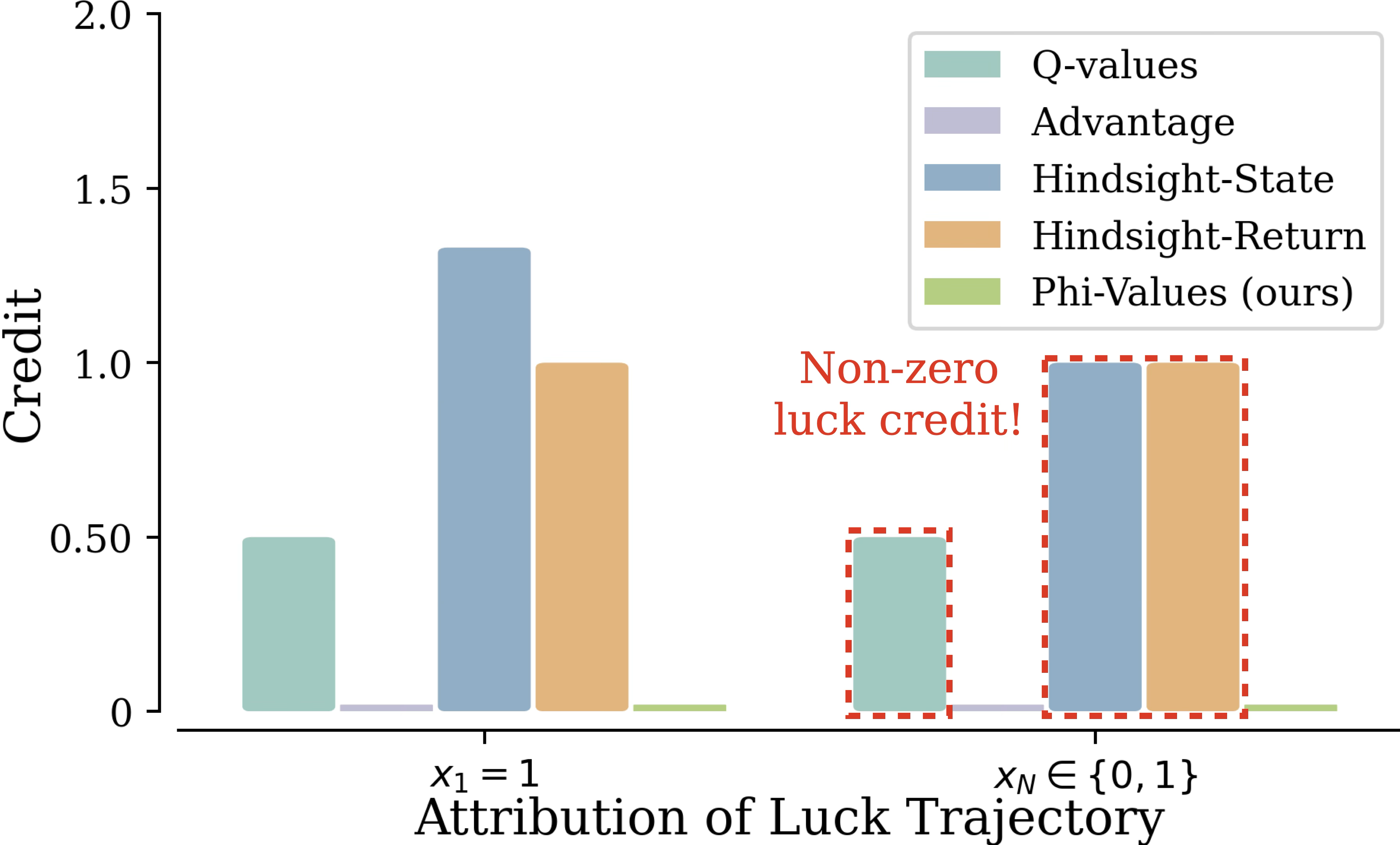}
    \end{subfigure}
    \caption{\textbf{Among all tested methods, only causal credit assignment correctly distinguishes Skill from Luck.} \emph{Left:} A SkillLuck MDP under a uniform random policy. The first action selects the branch; on the top branch, $x_N{=}1$ yields reward $1$ with certainty (``Skill,'' red) while $x_N{=}0$ yields $0$ (gray). On the bottom branch, the reward is drawn from $\text{Bern}(0.5)$ regardless of $x_N$ (``Luck,'' blue). \emph{Center:} Only our causal credit assignment assigns more credit to the first action on the Skill trajectory ($x_1{=}0$) than on the Luck trajectory ($x_1{=}1$). \emph{Right:} On the Luck trajectory, $x_N$ has no causal effect on the reward. Only the advantage and our method correctly assign zero credit to $x_N$.}
    \label{fig:intro example}
    \vspace{-0.1in}
\end{figure*}

However, locating the true causes of the rewards is challenging~\citep{pignatelli2024creditsurvey}.
In Figure~\ref{fig:intro example}, we construct a simple MDP to reflect this intricate nature of CAP. 
Under a uniform random policy, the agent's first action determines which branch it follows.
On the top branch, the last action determines the reward: one choice yields reward $1$ with certainty (``Skill,'' red) while the other yields $0$ (gray).
On the bottom branch, the reward is $0$ or $1$ with equal probability regardless of the last action (``Luck,'' blue).
Thus, a good credit assignment method should assign bigger positive credit to both the first and last actions on Skill trajectory than those on Luck, since changing either one would have altered the return distribution.
On the Luck trajectory, the last action has no causal effect on the return and should receive zero credit. Surprisingly, as we show in Figure~\ref{fig:intro example}, none of the prominent existing methods produces the fully correct credit assignment
(see Appendix~\ref{app:ca-methods} for a detailed discussion).

Humans, on the other hand, routinely perform credit assignment which does not seem to be challenging as it appears to machines.
One possible reason is that we are endowed with causal reasoning capabilities that enable us to allocate credit counterfactually.
Chess masters, for instance, replay key positions after every game, analyzing what would have happened under different moves.
As \citet{kasparov2007life} writes, ``it is so important to question success as vigorously as you question failure.''
Structural causal models (SCMs)~\citep{pearl2009causality,bareinboim2022pearl} formalize this type of counterfactual reasoning: given an observed trajectory, we can simulate the outcome under a different action at any step while holding all other factors fixed, isolating the causal effect of that decision. This type of counterfactual analysis lays the foundation for a fine-grained attribution for individual actions in each trajectory regardless of delayed effect, stochasticity and sparsity of the reward signals. 

In this work, building on the Counterfactual Shapley Value ($\phi$-value) framework of \citet{lee2025}, we distribute the total causal effect across actions using Shapley values~\citep{shapley:book1952}, which provide the unique allocation satisfying efficiency, symmetry, and null-player axioms.
The resulting $\phi$-MDP replaces the environment's reward signal with per-step $\phi$-values, providing dense causal learning signals even when the original reward is delayed or sparse.
We prove this redistribution preserves the optimal policy (Theorem~\ref{thm:optimal-policy-equiv}) and derive a consistent estimator computing all $T$ values from a single coalition sample in $O(T)$ time (Theorem~\ref{thm:optimal-proposal}).
We integrate these into $\phi$-PPO, which replaces per-step rewards with $\phi$-values and uses Prioritized Trajectory Replay (PTR).
PTR selects subtrajectories by causal importance, in contrast to PER~\citep{schaul2015per} which prioritizes individual transitions by TD error.
Our contributions are summarized as follows.
\begin{itemize}[label=\textbullet, leftmargin=20pt, topsep=0pt, parsep=0pt, itemsep=1pt]
    \item \textbf{Credit Assignment via $\phi$-values}. We show that credit assignment via $\phi$-values is equivalent to a reward redistribution that
    preserves the optimal policy and spikes only on causal actions.
    \item \textbf{Tractable $\phi$-value Estimation}. We propose a consistent $\phi$-value estimator that achieves linear time complexity, with a tunable bias--variance tradeoff via $\lambda$-bootstrapping.
    \item \textbf{Efficient and Explainable Policy Optimization}. We propose $\phi$-PPO and demonstrate its strong empirical performance in several challenging credit assignment testbeds. 
\end{itemize}

%% file: figures/intro_credit_assignment_example.tex

\begin{tikzpicture}[
    state/.style={inner sep=0pt, minimum size=7mm, font=\normalsize, circle, draw},
    start/.style={state, fill=white},
    mid/.style={state, fill=gray!10, draw=gray!40},
    term/.style={state, fill=yellow!30, draw=orange!50, },
    arr/.style={->, thick, >=stealth},
    opt/.style={arr, red!50},
    sub/.style={arr, blue!70},
    lbl/.style={font=\tiny, midway},
    ann/.style={font=\tiny, text=black!60},
]

\node[start] (s0) at (-0.2, 0) {$s_1$};

\node[mid] (s1t) at (0.8, 1.1) {$s_{2}$};
\node[ann] (dott) at (1.7, 1.1) {$\cdots$};
\node[mid] (snt) at (2.6, 1.1) {$s_{N}$};
\node[term] (Yt) at (4.1, 1.1) {$s_F$};

\node[mid] (s1b) at (0.8, -1.1) {{$s_{2}'$}};
\node[ann] (dotb) at (1.7, -1.1) {$\cdots$};
\node[mid] (snb) at (2.6, -1.1) {{$s_{N}'$}};
\node[term] (Yb) at (4.1, -1.1) {$s_F'$};

\draw[opt] (s0) -- node[lbl, above, sloped] {\small $x_1{=}0$} (s1t);
\draw[sub] (s0) -- node[lbl, below, sloped] {\small $x_1{=}1$} (s1b);
\draw[arr, gray!50] (s1t) -- (dott);
\draw[arr, gray!50] (dott) -- (snt);
\draw[arr, gray!50] (snt) -- node[lbl, above] {\normalsize $\substack{x_{N}{=}0\\y=0}$} (Yt);
\draw[arr, gray!50] (s1b) -- (dotb);
\draw[arr, gray!50] (dotb) -- (snb);
\draw[arr, blue!70] (snb) -- node[lbl, below] {\normalsize $\substack{\forall x_{N},\\y\sim \operatorname{Bern}(0.5)}$} node[lbl, above] {\scriptsize ``Luck"} (Yb);
\draw[arr, red!50] (snt) -- node[lbl, below, sloped] {\normalsize $\substack{x_{N}{=}1\\y=1.0}$} node[lbl, above, sloped] {\scriptsize ``Skill"} (Yb);


\end{tikzpicture}

%% file: sections/2-prelims.tex
\section{Preliminaries}
\label{sec:prelim}

In this section, we establish foundational concepts to facilitate our discussion. Throughout the paper, we use uppercase letters to denote random variables ($X$), lowercase for their realizations ($x$) and bold letters for sets ($\mathbf{V}$).

\subsection{Structural Causal Models and MDPs}
\label{sec:scm}

Structural Causal Models (SCMs) lay the mathematical foundations for counterfactual reasoning needed to isolate causal contributions \citep{bareinboim2022pearl}.

\begin{definition}[Structural Causal Model]
\label{def:scm}
An SCM is $\cM = \langle \V, \U, \mathcal{F}, P(\U) \rangle$ where $\V$ is endogenous variables, $\U$ is exogenous variables, $\mathcal{F} = \{f_i\}$ are structural equations $V_i = f_i(\mathbf{Pa}_i, \U_i)$, and $P(\U)$ is exogenous distribution.
\end{definition}

Under this definition, the counterfactual $V_{\mathbf{x}}(\mathbf{u})$ is the value of $V$ when $\X$ is set to $\mathbf{x}$ while $\U = \mathbf{u}$.
Then we cast MDPs as SCMs to ground the credit assignment discussion with causal semantics.

\begin{definition}[MDP-SCM]
\label{def:mdp-scm}
An MDP-SCM is an SCM with endogenous variables $\V = \{S_t, X_t, Y_t\}_{t=1}^T \cup \{Y\}$ (states, actions, per-step rewards, and outcome $Y = f_Y(Y_{1:T})$), mutually independent exogenous variables $\U = \{U_{S_t}, U_{X_t}, U_{Y_t}\}_{t=1}^T$, and structural equations defining system dynamics:
$
    S_1 = f_S(U_{S_1}), S_{t+1} = f_S(S_t, X_t, U_{S_{t+1}}),
    X_t = \pi(S_t, U_{X_t}),\ Y_t = r(S_t, X_t, U_{Y_t}) 
$.
\end{definition}
We use $P(S)$ to denote the initial state distribution induced by $U_{S_1}$ and the discounted total reward as the outcome $Y = \sum_t \gamma^{t-1} Y_t$. Value functions $V(s)$, Q-values $Q(s,x)$, and advantages $A(s,x)$ under MDP-SCM follow the standard RL definitions~\citep[Sec. 3.1]{sutton2018}.

\subsection{Policy Gradients}
\label{sec:policy gradients}
The goal of reinforcement learning is to find the policy that maximizes the cumulative reward (outcome),
$
    \pi^* = \arg\max_\pi \mathbb{E}_{\tau\sim \mathcal{M}^\pi}[Y],
$
where trajectory $\tau\sim \mathcal{M}^\pi$ is sampled from the environment MDP $\mathcal{M}^\pi$ under policy $\pi$. Optimizing this objective with respect to a parameterized policy $\pi_\theta$ leads to the basic form of policy gradients~\citep{williams1992simple},
\begin{align}
    \mathbb{E}\Big[\sum_{t \geq 0} \gamma^t\nabla_\theta \log\pi(x_t|s_t) \big(\sum_{t'\geq t} \gamma^{t'-t}Y_{t'} - V(s_t)\big)\Big]
\end{align}
where $V(s_t) = \mathbb{E}\big[\sum_{t'\geq t} \gamma^{t'-t}Y_{t'}|s_t\big]$ is the value function. Here it serves as a baseline with respect to which the advantage of taking the specific action $x_t$ instead of others is strengthened. Subtracting the baseline will not bias the gradient as long as the baseline is a function of state $s_t$. Modern policy gradient methods use generalized advantages~\citep{schulmanHighDimensionalContinuousControl2016a} and gradient clipping~\citep{schulman2017ppo} to reduce variance.
In this work, we base our learning algorithm on proximal policy gradient (PPO)~\citep{schulman2017ppo}, whose objective is to maximize the following,
\begin{align}
    \min\bigl(r_t(\theta) \hat{A}_t,\; \mathrm{clip}(r_t(\theta), 1{-}\epsilon_{\mathrm{clip}}, 1{+}\epsilon_{\mathrm{clip}}) \hat{A}_t\bigr)
\end{align}
where $\hat{A}_t =  \sum_{\ell=0}^{T-t+1} (\gamma\lambda)^\ell \delta_{t+\ell},\  \delta_{t} = Y_t + \gamma V(s_{t+1}) - V(s_t)$ is the generalized advantage and $r_t(\theta)=\frac{\pi_\theta(x_t|s_t)}{\pi_{\theta_{\operatorname{old}}}(x_t|s_t)}$ is the importance ratio between the current policy distribution and the previous policy collecting the data.

\subsection{Counterfactual Shapley Values}
\label{sec:l3-sv}
Shapley values~\citep{shapley1953} fairly distribute credit among players in cooperative games.
Given players $\X$ and value function $f: 2^\X \to \R$, a \emph{coalition} $\mathbf{Z} \subseteq \X$ is a subset of players, and $f(\mathbf{Z})$ measures their joint contribution.
The Shapley value $\phi_t$ quantifies player $X_t$'s fair share of the total value $f(\X) - f(\emptyset)$:
\begin{align}
    \phi_t = \sum_{\mathbf{Z} \subseteq \X \setminus \{X_t\}} \frac{|\mathbf{Z}|!(T-|\mathbf{Z}|-1)!}{T!} \bigl[f(\mathbf{Z} \cup \{X_t\}) - f(\mathbf{Z})\bigr].
\end{align}
Shapley values satisfy desirable properties~\citep{shapley1953}, including \emph{efficiency}: $\sum_t \phi_t = f(\X) - f(\emptyset)$.
For credit assignment, actions $\X = \{X_1, \ldots, X_T\}$ are players, and given trajectory $\tau = (s_{1:T}, x_{1:T}, y_{1:T})$ (the \emph{evidence} for counterfactual reasoning), the Shapley value $\phi_t$ quantifies how much $X_t$ contributed to outcome $Y$.

\paragraph{Causal Contributions.}
We use counterfactual natural total effect~\citep{lee2025} to measure a set of actions' causal contributions to the outcome. It is the expected difference between the observed outcome and a counterfactual outcome where a subset of actions are replaced.
``Natural'' means the comparison uses the posterior $P(\U \mid \tau)$ rather than the prior, ensuring the counterfactual is consistent with the observed trajectory:
\begin{align}
    \label{eq:ctf nte}
    f(\mathbf{Z}) &= \mathrm{NTE}(\mathbf{Z}, Y \mid \tau) 
    = \E_{\substack{\mathbf{u'} \sim P(\U),  \mathbf{u} \sim P(\U \mid \tau)}}\bigl[Y(\mathbf{u}) - Y_{\mathbf{Z}(\mathbf{u}')}(\mathbf{u})\bigr],
\end{align}
where $Y(\mathbf{u})$ is the observed outcome and $Y_{\mathbf{Z}(\mathbf{u}')}(\mathbf{u})$ is the counterfactual with coalition actions replaced by baseline alternatives $\mathbf{Z}(\mathbf{u}')$ sampled from a baseline policy $\pi_{\mathrm{base}}$ (Sec.~\ref{sec:ctf-sim}), while exogenous $\mathbf{u}$ is fixed.

\paragraph{Kernel Formulation.}
The standard Shapley formula sums over $T!$ permutations; the kernel formulation~\citep{lundberg2017} sums over $2^T$ coalitions, enabling amortization (Section~\ref{sec:computing}).
Let $\mathbf{z} \in \{0,1\}^T$ denote the coalition mask ($z_t = 1$ iff $X_t \in \mathbf{Z}$):
\begin{equation}
    \phi_t = \sum_{\mathbf{z} \in \{0,1\}^T} \kappa_t(\mathbf{z}) \cdot f(\mathbf{z}),
    \label{eqn:kernel}
\end{equation}
where $|\mathbf{z}| = \sum_t z_t$ and $
    \kappa_t(\mathbf{z}) = 
        \frac{1}{T\binom{T-1}{|\mathbf{z}|-1}}$ if $ z_t = 1$ while 
        $\kappa_t(\mathbf{z}) = -\frac{1}{T\binom{T-1}{|\mathbf{z}|}}$ if $ z_t = 0$.

%% file: sections/3-credit-assignment.tex
\section{Counterfactual Credit Assignment}
\label{sec:credit-assignment}

\input{figures/overview}

In the remainder of the text, we introduce our proposed causal solution to CAP. In Section~\ref{sec:credit-assignment}, we formalize the credit assignment problem from a causal perspective, propose $\phi$-values as a principled solution and its related calculations in the RL context. In Section~\ref{sec:computing}, we develop efficient estimators for $\phi$-values and incorporate them into policy learning.
Figure~\ref{fig:phi-mdp-overview} visualizes the conceptual structures.

First, we define the Credit Assignment Problem (CAP) as a causal inference problem: only actions that causally contribute to the outcome should receive credit, proportional to their contributions.

\begin{definition}[Causal Credit Assignment, Informal]
\label{def:causal credit assignment}
A causal credit assignment function is a mapping $\phi$ that maps a trajectory, an action at a time step and a baseline policy to a real number denoting its causal contribution to the outcome observed in the trajectory,
\begin{align}
    \phi: \tau \times \mathcal{X} \times T \times \Pi \mapsto \mathbb{R}^T
\end{align}
where $\tau \sim \cM_\pi$ is a trajectory sampled from the baseline policy and $\Pi$ is the policy space. The causal assignment $\phi$ should satisfy the following desiderata,
\begin{enumerate}[label={$\mathbf{D}_\arabic*$}, leftmargin=23pt, topsep=0pt, parsep=0pt, itemsep=1pt]
    \item \textbf{Causal Admissibility} : Non-cause actions should be assigned zero credit;
    \item \textbf{Causal Power} : Actions causing the rewards should be assigned non-zero credit;
    \item \textbf{Causal Normality} : Actions against more probable baseline policies get more credit;
    \item \textbf{Causal Effect Scaling} : Given the same baseline policy, actions affecting the outcome more should be allocated more credit. 
\end{enumerate}
\end{definition}

A complete formal version of the causal credit assignment desiderata is presented in App.~\ref{sec:theory detail}, Def.~\ref{def:causal credit assignment formal}. Here, we measure each action's causal contribution to the outcome as the counterfactual natural total effect (Eq.~\ref{eq:ctf nte}). Unlike discounted summation of rewards, the counterfactual natural total effect is the difference between the returns under baseline policy and returns had we acted differently given an already realized trajectory. Advantage function shares a similar idea but they cannot capture counterfactual outcomes with respect to specific trajectories or quantify an action's full potential under different action selections at other time steps. 

Counterfactual natural total effect provides a fair source of credit pool as the basis for allocation and these desiderata are axioms embedded in human credit assignment systems, which a causal credit assignment function should satisfy. $\mathbf{D_1,D_2}$ together guarantee that the credit assignment function generates both sufficient and necessary allocations. That is, actions get credit if and only if they cause the outcome. $\mathbf{D_3,D_4}$ require that the scale of the allocated credit should reflect an action's relative contribution compared against a baseline policy. We show that counterfactual Shapley values ($\phi$-values, Eq.~\ref{eqn:kernel}) calculated from counterfactual natural total effect satisfy the above desiderata~\citep{lee2025} and is a desirable credit assignment method. 

\begin{restatable}[$\phi$-Values are Causal]{theorem}{ThmPhiIsCausal}
\label{thm:phi is causal}
    $\phi$-values satisfy the causal credit assignment criteria, $\mathbf{D_1-D_4}$.
\end{restatable}

Next, we present the details in calculating $\phi$-values and $\phi$-MDP, an MDP that redistributes rewards according to Counterfactual Shapley values, as our solution to the Causal Credit Assignment Problem.

\subsection{Counterfactual Simulation}
\label{sec:ctf-sim}
\input{figures/ctfsim}
To compute Shapley values $\phi_t$, we need a causal contribution function $f: 2^\X \to \R$ measuring $\X$'s contribution under a given coalition $\mathbf{z}$.
The NTE from Section~\ref{sec:l3-sv} averages over both posterior and prior exogenous variables; here, we condition on the observed trajectory, giving the \emph{conditional} causal contribution:
$
    f(\mathbf{z}) = Y - Y_{\sigma_{\mathbf{z}}},
$
where $Y$ is the observed return and $Y_{\sigma_{\mathbf{z}}}$ is the counterfactual return under intervention $\sigma_{\mathbf{z}}$
which specifies counterfactual actions based on coalition membership $\mathbf{z} \in \{0,1\}^T$ and trajectory state:
\begin{align}
    \sigma_{\mathbf{z}}(t) = \begin{cases}
        x_t & z_t = 0,\ s^{\mathbf{z}}_t = s_t \\
        x \sim \pi(\cdot \mid s^{\mathbf{z}}_t) & z_t = 0,\ s^{\mathbf{z}}_t \neq s_t \\
        x \sim \pi_{\mathrm{base}}(\cdot \mid s^{\mathbf{z}}_t) & z_t = 1
    \end{cases}
    \label{eqn:sigma-z}
\end{align}
where $s^{\mathbf{z}}_t$ is the counterfactual state at time $t$.
The baseline $\pi_{\mathrm{base}}$ defines the ``default'' actions for intervened time steps.
Options include: uniform random ($\pi_{\mathrm{base}}(x \mid s) = 1/|\cA|$), measuring contribution relative to random behavior; a fixed reference policy $\pi_0$, measuring improvement over $\pi_0$; or self-baseline ($\pi_{\mathrm{base}} = \pi$), isolating whether the \emph{specific} action $x_t$ matters versus what the same policy typically does.
We use self-baseline throughout training (Section~\ref{sec:doorkey} illustrates the effect of alternative baselines on attributions).
The intervention design rests on the following assumption.

\begin{assumption}[Counterfactual Independence]
\label{assum:ctf-indep}
Exogenous noise is independent across steps: $(S_{t+1})_{S_t, X_t} \perp (S_{t+1})_{S'_t, X'_t}$ for $(S_t, X_t) \neq (S'_t, X'_t)$, and $(X_t)_{S_t} \perp (X_t)_{S'_t}$ for $S_t \neq S'_t$.
\end{assumption}

This assumption does not constrain the MDP, only its SCM representation.
Since any distribution can be sampled via inverse CDF with independent uniform noise, any MDP admits an SCM with independent exogenous variables and the same optimal policy; the assumption selects this representation.
The consequence is simple: \emph{when parents match the observed trajectory, the observed value is the counterfactual; when parents differ, we resample}.
For actions (parent $S_t$): when $z_t = 0$ and $s^{\mathbf{z}}_t = s_t$, the counterfactual equals the observed $x_t$; when $s^{\mathbf{z}}_t \neq s_t$, we resample from $\pi(\cdot \mid s^{\mathbf{z}}_t)$.
When $z_t = 1$, we resample from $\pi_{\mathrm{base}}$.
For transitions and rewards (parents $(S_t, X_t)$): when $(s^{\mathbf{z}}_t, x^{\mathbf{z}}_t) = (s_t, x_t)$, we reuse $(s_{t+1}, y_t)$; otherwise, we resample.
See Fig.~\ref{fig:ctf-sim} for a visualization of the counterfactual simulation process.

The full pseudo-code is provided in Algo.~\ref{alg:ctf-sim} in appendix Sec.~\ref{sec:full code}; Section~\ref{sec:computing} makes exact computation tractable via coalition sampling.

\subsection{Reward Redistribution}
\label{sec:phi-mdp-def}
With the counterfactual Shapley values, we can define a new rewarding system based on $\phi_t$ while keeping the system dynamics unchanged. This leads to the $\phi$-MDP, a reward-redistributed counterpart of the original MDP.

\begin{definition}[$\phi$-MDP]
\label{def:phi-mdp}
Given an MDP-SCM $\cM$ and baseline policy $\pi_{\mathrm{base}}$, the \emph{$\phi$-MDP} $\cM^\phi_\pi$ replaces the reward function $r$ with the \emph{policy-induced Shapley reward}:
\begin{align}
    r^\phi_\pi(s, x) = \E_{\tau \sim \pi}[\phi_t \mid S_t = s, X_t = x],
\end{align}
where the expectation is over trajectories $\tau$ sampled from $\pi$, conditioned on visiting $(S_t, X_t) = (s, x)$.
\end{definition}

\begin{restatable}[Gradient Equivalence]{lemma}{LemGradientEquiv}
\label{lem:gradient-equiv}
For any baseline policy $\pi_{\mathrm{base}}$:
$\E_\pi[\sum_t \phi_t \cdot \nabla \log \pi(X_t \mid S_t)] = \E_\pi[Y \cdot \sum_t \nabla \log \pi(X_t \mid S_t)]$.
\end{restatable}

\begin{remark}[Independence under self-baseline]
\label{rem:self-baseline-indep}
The proof uses $Y' \perp X_t \mid S_t$ where $Y' = Y_{\sigma_\mathbf{1}}$ is the full-baseline outcome.
Under self-baseline, $Y'$ is computed on a \emph{counterfactual trajectory} using fresh action samples $X'_{1:T} \sim \pi(\cdot \mid S^{\mathbf{1}}_{1:T})$, where $S^{\mathbf{1}}_{1:T}$ denotes the counterfactual state sequence.
Crucially, $Y'$ carries no information about which specific action $X_t$ was taken at the observed state $S_t$: the counterfactual uses fresh action samples independent of $X_t$.
The score function identity $\E_{X_t \sim \pi}[\nabla \log \pi(X_t \mid S_t)] = 0$ then gives $\E[Y' \nabla \log \pi_t] = 0$.
\end{remark}

Gradient equivalence means policy gradient algorithms receive the same update direction in expectation whether maximizing returns from MDP $\cM$ or its dual $\phi$-MDP.

\begin{restatable}[Optimal Policy Equivalence]{theorem}{ThmOptimalPolicyEquiv}
\label{thm:optimal-policy-equiv}
$\phi$-MDP preserves optimal policies as the original MDP.
\end{restatable}

More specifically, $
\arg\max_\theta \E_{\pi_\theta}[\sum_t \phi_t] = \arg\max_\theta \E_{\pi_\theta}[Y]$, so $\phi$-MDP is a drop-in replacement: any policy gradient algorithm applied to the redistributed rewards converges to the same optimal policy.
One caveat: the $\phi$-MDP is \emph{undiscounted} ($\gamma_\phi = 1$).
Since the NTE game outcome $Y = \sum_t \gamma^{t-1} r_t$ already incorporates temporal discounting, applying $\gamma$ to $\phi_t$ would double-discount.
The undiscounted $\phi$-return $\sum_t \phi_t$ equals $Y - Y^{\mathbf{1}}$ by Shapley efficiency.

\subsection{How Causal Credit Assignment Helps Learning}
\label{sec:when-helps}

We have shown that the $\phi$-MDP provides a principled credit assignment while preserving the optimal policy, but does it also improve learning efficiency?
As an overview, Fig.~\ref{fig:three-dimensions} illustrates each mechanism visually.
Below, we present the formal conditions our method helps policy optimization under the assumption of a finite second moment.

\begin{assumption}[Finite Second Moment]
\label{assum:second-moment}
$\forall t, \E[Y_t^2] \leq \sigma^2$.
\end{assumption}

\input{figures/three-dimensions}

\begin{restatable}[Sparse Causality]{proposition}{PropSparseCausality}
\label{prop:sparse-causality}
Let $k' = |\{t : \phi_t \neq 0\}|$, $\sigma^2_\phi = \max_t \Var[\phi_t \nabla\!\log \pi_t]$, and $\bar{\rho} \in [0,1]$ the maximal pairwise correlation among causal terms.
Then $\Var[\hat{g}^\phi] \leq k'(1 + (k'{-}1)\bar{\rho})\, \sigma^2_\phi$.
\end{restatable}

When few actions influence the outcome ($k' \ll T$), $\phi$-values assign zero credit to non-causal actions, so only $k'$ terms contribute to gradient variance.
The cross-correlation $\bar{\rho}$ is small in practice: the score function identity $\E[\nabla\!\log\pi_s \mid S_s] = 0$ decorrelates terms at distinct timesteps, and $\bar{\rho} = 0$ exactly in deterministic environments (Appendix~\ref{sec:theory detail}).

\begin{restatable}[High Stochasticity]{proposition}{PropStochasticity}
\label{prop:stochasticity}
Let $\mathbf{z}' = \mathbf{z}$ with $z_t = 1$ and $\rho_Y = \Corr[Y^{\mathbf{z}'}, Y^{\mathbf{z}}]$ under shared exogenous randomness.
Then $\Var[\phi_t] \propto (1-\rho_Y)$.
\end{restatable}

Proposition~\ref{prop:stochasticity} holds in expectation over coalitions $\mathbf{z} \sim Q^*$ and trajectories.
Define the \emph{effective coupling} $\bar{\rho}_Y = \E_{\mathbf{z} \sim Q^*}[\Corr[Y^{\mathbf{z}'}, Y^{\mathbf{z}}]]$, averaging over the optimal coalition distribution.
When a counterfactual trajectory diverges (off-track), fresh noise is sampled, reducing $\Corr[Y^{\mathbf{z}'}, Y^{\mathbf{z}}]$ for that sample.
The effective coupling $\bar{\rho}_Y$ is high when: (1) coalitions are small; (2) transitions are nearly deterministic; or (3) rewards depend on shared exogenous factors (unobserved confounders).
Variance reduction is proportional to $(1 - \bar{\rho}_Y)$.

\begin{restatable}[Direct Propagation]{proposition}{PropTDPropagation}
\label{prop:td-propagation}
Tabular TD(0) with terminal-only reward requires $O(T)$ episodes for credit to reach the initial state.
$\phi$-redistribution provides credit in $O(1)$ episodes via $O(T M_{\delta,\epsilon})$ counterfactual simulations.
\end{restatable}

\begin{restatable}[Propagation Complexity]{remark}{CorPropagationComplexity}
For terminal-only reward:
TD(0) requires $O(T)$ episodes and $O(T^2)$ updates;
$\phi$-redistribution requires $O(1)$ episodes and $O(T M_{\delta,\epsilon})$ counterfactual simulations.
\end{restatable}

This shows that $\phi$-values avoid the delayed-effect and inefficient propagation problem of TD(0) methods. Though n-step returns, TD($\lambda$) and the use of replay buffers mitigate this problem~\citep[Sec. 1.2-1.3]{daley2025multistep}, their core credit assumption rests on temporal proximity~\citep{pignatelli2024creditsurvey}. As the trajectory rolls out, reward in a distant step is deemed to be exponentially less likely caused by a given action (by the discounting factor). This clearly does not reflect an action's true causal contribution as we have demonstrated in Fig.~\ref{fig:intro example}. And we will also see in experiment Sec.~\ref{sec:experiments} that even with all the modern DRL techniques~\citep{schulman2017ppo,DBLP:journals/corr/SchaulQAS15prioritizedexpreplay}, existing algorithms still cannot solve those seemingly naive tabular MDPs due to the implicit time proximity assumption in credit assignment.
And the most important implication of Prop.~\ref{prop:td-propagation} is that our proposed $\phi$-value can allocate credit in a lossless and faithful way w.r.t each action's actual causal contribution. 

To sum, the ``sweet spot'' occurs when all three dimensions favor $\phi$-values: sparse causality ($k' \ll T$), high stochasticity ($\rho_Y \approx 1$), and delayed reward.
When both sparsity and stochasticity benefits apply, the variance reduction factors multiply: sparsity reduces contributing terms from $T$ to $k'$, while coupled comparison reduces each term's variance by $(1-\rho_Y)$, giving combined ratio $\frac{k'}{T}(1-\rho_Y)$.
On the other hand, 
there are cases when $\phi$-methods offer marginal advantage in terms of learning efficiency other than explanatory benefits over the standard TD methods when: (1) rewards are dense ($k' \approx T$); (2) environment is deterministic ($\rho_Y$ is high); or (3) actions are rewarded immediately.

%% file: figures/overview.tex

\begin{figure*}[t]
\centering
\vspace{-0.1in}
\begin{tikzpicture}[
    mainbox/.style={draw, rounded corners=5pt, minimum height=1.1cm, minimum width=2.6cm, font=\small, align=center},
    smallbox/.style={draw, rounded corners=3pt, minimum height=0.6cm, font=\scriptsize, align=center, fill=white},
    benefit/.style={draw=blue!30, rounded corners=3pt, minimum height=0.9cm, minimum width=2.6cm, font=\scriptsize, align=center, fill=blue!5},
    dataarrow/.style={-{Stealth[length=2.5mm]}, blue!50!black, thick},
    fadetarrow/.style={-{Stealth[length=2mm]}, black!30},
    connect/.style={-{Stealth[length=2mm]}, blue!40, densely dashed}
]


\node[mainbox, fill=white, draw=black!50] (mdp) at (0, 0)
    {\textbf{MDP} $M$\\[-2pt]{\scriptsize rewards $Y_t$}};

\node[mainbox, fill=blue!12, draw=blue!50!black, line width=0.8pt] (phimdp) at (8, 0)
    {\textbf{$\phi$-MDP} $\cM^\phi_\pi$\\[-2pt]{\scriptsize rewards $r^\phi_\pi$}};

\draw[dataarrow, line width=1.2pt] (mdp.east) --
    node[above, font=\footnotesize, black!80] {Redistribute Rewards}
    node[below, font=\footnotesize, black!80] {Section \ref{sec:phi-mdp-def}}
    (phimdp.west);


\begin{scope}[shift={(0, -2.5)}]
    \fill[black!3, rounded corners=6pt] (-2.2, -1.0) rectangle (10.2, 1.2);
    \node[font=\footnotesize\itshape, black!40, anchor=west] at (-2.1, 0.95) {Estimation};

    \node[smallbox, minimum width=1.8cm] (tau) at (0, -0.4) {$\tau \sim (M,\pi)$};
    \node[smallbox, minimum width=1.4cm] (tauz) at (4.0, -0.4) {$\tau^{\mathbf{z}}$};
    \node[smallbox, fill=blue!10, draw=blue!40, minimum width=1.4cm] (phi) at (8, -0.4) {$\hat{\phi}_t$};


    \draw[dataarrow] (tau.east) -- node[above, font=\footnotesize, black!60] {Counterfactual Sim.} node[below, font=\footnotesize, black!60] {Section \ref{sec:ctf-sim}} (tauz.west);
    \draw[dataarrow] (tauz.east) -- node[above, font=\footnotesize, black!60] {Credit Assignment} node[below, font=\footnotesize, black!60] {Section \ref{sec:computing}} (phi.west);

\end{scope}

\draw[fadetarrow] (mdp.south) -- (0, -2.55);

\draw[connect] (8, -2.55) -- (8, -0.5)
    node[pos=0.4, right, font=\footnotesize, blue!40] {estimates $r^\phi_\pi$};

\end{tikzpicture}
\caption{Our counterfactual Shapley values assigns causal credits to each action via reward redistribution.
\emph{Top:} The MDP $M$ is transformed into $\phi$-MDP $\cM^\phi_\pi$ by redistributing rewards $Y_t$ into $r^\phi_\pi(s,x) = \mathbb{E}[\phi_t \mid S_t{=}s, X_t{=}x]$.
Both MDPs share the same optimal policy (Theorem~\ref{thm:optimal-policy-equiv}).
\emph{Bottom:} The estimation pipeline samples a trajectory $\tau \sim \pi$ and a coalition mask $\mathbf{z} \sim Q^*$, then simulates the counterfactual trajectory $\tau^{\mathbf{z}}$ (Sec.~\ref{sec:ctf-sim}). Action credits are estimated based on sampled counterfactual trajectories (Algo.~\ref{alg:bootstrapped-estimator}) and used for policy optimization (Sec.~\ref{sec:computing}).
}
\vspace{-0.1in}
\label{fig:phi-mdp-overview}
\end{figure*}
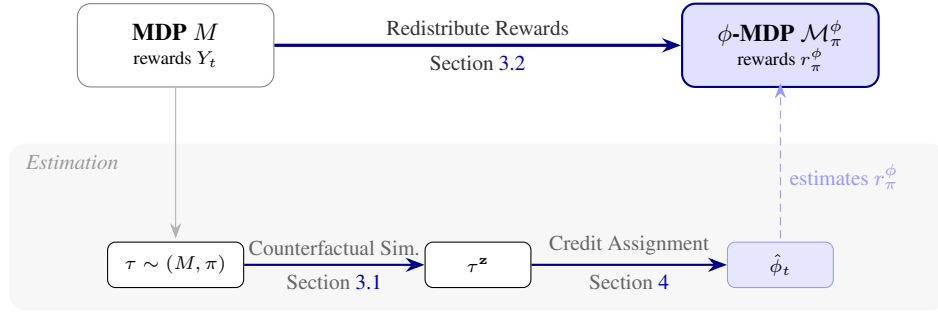

%% file: figures/ctfsim.tex

\begin{figure*}[t]
\centering
\vspace{-0.2in}
\begin{tikzpicture}[
    node distance=1.2cm,
    state/.style={circle, draw, minimum size=0.6cm, font=\scriptsize},
    action/.style={rectangle, draw, minimum size=0.5cm, font=\scriptsize},
    reuse/.style={fill=white},
    interv/.style={fill=blue!20},
    offtrack/.style={fill=gray!25},
    arrow/.style={-{Stealth[length=1.8mm]}, thick},
    label/.style={font=\scriptsize},
    tinylabel/.style={font=\tiny}
]
    \def\leftOffset{0}

    \def\colA{0.8}
    \def\colB{2.0}
    \def\colC{3.2}
    \def\colD{4.4}
    \def\colE{5.6}

    \def\rowZ{2.55}  
    \def\rowTraj{1.2}

    \node[label] at (3.2, \rowZ+0.8) {$\mathbf{z} = (0,0,0,0,0)$};

    \def\zShift{0}
    \node[tinylabel] at (\colA+\zShift, \rowZ) {$z_1\!=\!0$};
    \node[tinylabel] at (\colB+\zShift, \rowZ) {$z_2\!=\!0$};
    \node[tinylabel] at (\colC+\zShift, \rowZ) {$z_3\!=\!0$};
    \node[tinylabel] at (\colD+\zShift, \rowZ) {$z_4\!=\!0$};
    \node[tinylabel] at (\colE+\zShift, \rowZ) {$z_5\!=\!0$};

    \node[state, reuse] (Ls1) at (\colA, \rowTraj) {$s_1$};
    \node[action, reuse] (Lx1) at (\colA, \rowTraj+0.95) {$x_1$};
    \node[state, reuse] (Ls2) at (\colB, \rowTraj) {$s_2$};
    \node[action, reuse] (Lx2) at (\colB, \rowTraj+0.95) {$x_2$};
    \node[state, reuse] (Ls3) at (\colC, \rowTraj) {$s_3$};
    \node[action, reuse] (Lx3) at (\colC, \rowTraj+0.95) {$x_3$};
    \node[state, reuse] (Ls4) at (\colD, \rowTraj) {$s_4$};
    \node[action, reuse] (Lx4) at (\colD, \rowTraj+0.95) {$x_4$};
    \node[state, reuse] (Ls5) at (\colE, \rowTraj) {$s_5$};
    \node[action, reuse] (Lx5) at (\colE, \rowTraj+0.95) {$x_5$};

    \draw[arrow] (Ls1) -- (Lx1);
    \draw[arrow] (Ls2) -- (Lx2);
    \draw[arrow] (Ls3) -- (Lx3);
    \draw[arrow] (Ls4) -- (Lx4);
    \draw[arrow] (Ls5) -- (Lx5);

    \draw[arrow] (Ls1) -- (Ls2);
    \draw[arrow] (Ls2) -- (Ls3);
    \draw[arrow] (Ls3) -- (Ls4);
    \draw[arrow] (Ls4) -- (Ls5);

    \draw[arrow] (Lx1) to[out=0, in=180] (Ls2);
    \draw[arrow] (Lx2) to[out=0, in=180] (Ls3);
    \draw[arrow] (Lx3) to[out=0, in=180] (Ls4);
    \draw[arrow] (Lx4) to[out=0, in=180] (Ls5);

    \node[tinylabel] at (\colA, 0.5) {\textcolor{green!50!black}{on}};
    \node[tinylabel] at (\colB, 0.5) {\textcolor{green!50!black}{on}};
    \node[tinylabel] at (\colC, 0.5) {\textcolor{green!50!black}{on}};
    \node[tinylabel] at (\colD, 0.5) {\textcolor{green!50!black}{on}};
    \node[tinylabel] at (\colE, 0.5) {\textcolor{green!50!black}{on}};

    \node[tinylabel] at (3.2, 0.2) {$Y^{\mathbf{z}} = Y$ (all reused)};

    \def\rightOffset{6.5}

    \node[label] at (\rightOffset+3.2, \rowZ+0.8) {$\mathbf{z} = (0,1,0,1,1)$};

    \node[tinylabel] at (\rightOffset+\colA+\zShift, \rowZ) {$z_1\!=\!0$};
    \node[tinylabel, text=blue!70!black] at (\rightOffset+\colB+\zShift, \rowZ) {$z_2\!=\!1$};
    \node[tinylabel] at (\rightOffset+\colC+\zShift, \rowZ) {$z_3\!=\!0$};
    \node[tinylabel, text=blue!70!black] at (\rightOffset+\colD+\zShift, \rowZ) {$z_4\!=\!1$};
    \node[tinylabel, text=blue!70!black] at (\rightOffset+\colE+\zShift, \rowZ) {$z_5\!=\!1$};

    \node[state, reuse] (Rs1) at (\rightOffset+\colA, \rowTraj) {$s_1$};
    \node[action, reuse] (Rx1) at (\rightOffset+\colA, \rowTraj+0.95) {$x_1$};
    \node[state, reuse] (Rs2) at (\rightOffset+\colB, \rowTraj) {$s_2$};
    \node[action, interv] (Rx2) at (\rightOffset+\colB, \rowTraj+0.95) {$x^{\mathbf{z}}_2$};
    \node[state, offtrack] (Rs3) at (\rightOffset+\colC, \rowTraj) {$s^{\mathbf{z}}_3$};
    \node[action, offtrack] (Rx3) at (\rightOffset+\colC, \rowTraj+0.95) {$x^{\mathbf{z}}_3$};
    \node[state, reuse] (Rs4) at (\rightOffset+\colD, \rowTraj) {$s_4$};
    \node[action, interv] (Rx4) at (\rightOffset+\colD, \rowTraj+0.95) {$x^{\mathbf{z}}_4$};
    \node[state, offtrack] (Rs5) at (\rightOffset+\colE, \rowTraj) {$s^{\mathbf{z}}_5$};
    \node[action, interv] (Rx5) at (\rightOffset+\colE, \rowTraj+0.95) {$x^{\mathbf{z}}_5$};

    \draw[arrow] (Rs1) -- (Rx1);
    \draw[arrow] (Rs3) -- (Rx3);

    \draw[arrow] (Rs1) -- (Rs2);
    \draw[arrow] (Rs2) -- (Rs3);
    \draw[arrow] (Rs3) -- (Rs4);
    \draw[arrow] (Rs4) -- (Rs5);

    \draw[arrow] (Rx1) to[out=0, in=180] (Rs2);
    \draw[arrow] (Rx2) to[out=0, in=180] (Rs3);
    \draw[arrow] (Rx3) to[out=0, in=180] (Rs4);
    \draw[arrow] (Rx4) to[out=0, in=180] (Rs5);

    \node[tinylabel] at (\rightOffset+\colA, 0.5) {\textcolor{green!50!black}{on}};
    \node[tinylabel] at (\rightOffset+\colB, 0.5) {\textcolor{green!50!black}{on}};
    \node[tinylabel] at (\rightOffset+\colC, 0.5) {\textcolor{red!70!black}{off}};
    \node[tinylabel] at (\rightOffset+\colD, 0.5) {\textcolor{green!50!black}{on}};
    \node[tinylabel] at (\rightOffset+\colE, 0.5) {\textcolor{red!70!black}{off}};

    \node[tinylabel] at (\rightOffset+3.2, 0.2) {$Y^{\mathbf{z}} \neq Y$ (diverged)};

    \def\legendX{13.0}
    \node[action, reuse] at (\legendX, 2.6) {};
    \node[tinylabel, right] at (\legendX+0.2, 2.6) {reuse};
    \node[action, interv] at (\legendX, 1.9) {};
    \node[tinylabel, right] at (\legendX+0.2, 1.9) {intervention};
    \node[action, offtrack] at (\legendX, 1.2) {};
    \node[tinylabel, right] at (\legendX+0.2, 1.2) {off-track};
\end{tikzpicture}
\caption{Counterfactual simulation ($T=5$).
\emph{Left:} $\mathbf{z} = \mathbf{0}$, all values reused, $Y^{\mathbf{z}} = Y$.
\emph{Right:} $\mathbf{z} = (0,1,0,1,1)$; intervention at $t\!=\!2$ causes divergence, returns on-track at $t\!=\!4$ ($s^{\mathbf{z}}_4 = s_4$ by chance), diverges again at $t\!=\!5$.}
\vspace{-0.1in}
\label{fig:ctf-sim}
\end{figure*}
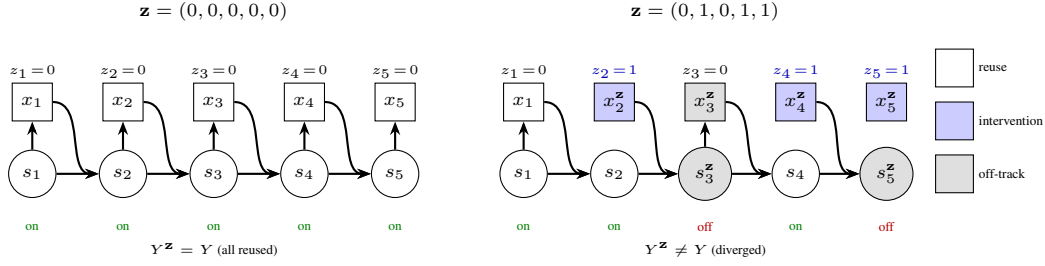

%% file: figures/three-dimensions.tex

\begin{figure*}[t]
\centering
\vspace{-0.15in}
\begin{tikzpicture}[
    tdcolor/.style={red!60},
    phicolor/.style={blue!60},
    tdbar/.style={fill=red!25, draw=red!50},
    phibar/.style={fill=blue!25, draw=blue!50},
    noncausal/.style={fill=gray!15, draw=gray!30},
    paneltitle/.style={font=\small\bfseries},
    annot/.style={font=\scriptsize, black!60},
    rowlabel/.style={font=\scriptsize\bfseries, black!70},
]

\def\pw{4.2}
\def\ph{2.6}
\def\gap{0.3}

\begin{scope}[shift={(0,0)}]
    \fill[black!3, rounded corners=3pt] (-0.7, -0.7) rectangle (\pw, \ph+0.4);
    \node[paneltitle] at (\pw/2-0.2, \ph+0.15) {(a) Sparse Causality};

    \node[rowlabel, anchor=east] at (0, 1.7) {TD};
    \foreach \i in {1,...,8} {
        \fill[tdbar] ({0.05 + (\i-1)*0.5}, 1.4) rectangle ++(0.4, 0.6);
    }
    \node[annot, tdcolor, anchor=west] at (4.15, 1.7) {all $T$};

    \node[rowlabel, anchor=east] at (0, 0.65) {$\phi$};
    \foreach \i in {1,2,4,5,6,8} {
        \fill[noncausal] ({0.05 + (\i-1)*0.5}, 0.4) rectangle ++(0.4, 0.15);
    }
    \foreach \i in {3,7} {
        \fill[phibar] ({0.05 + (\i-1)*0.5}, 0.4) rectangle ++(0.4, 0.6);
    }
    \node[annot, phicolor, anchor=west] at (4.15, 0.65) {only $k'$};

    \node[annot, text width=4cm, align=center] at (\pw/2-0.2, -0.35)
        {$\phi_t = 0$ for non-causal actions};
\end{scope}

\begin{scope}[shift={(\pw + \gap, 0)}]
    \fill[black!3, rounded corners=3pt] (-0.7, -0.7) rectangle (\pw, \ph+0.4);
    \node[paneltitle] at (\pw/2-0.3, \ph+0.15) {(b) High Stochasticity};

    \node[rowlabel, anchor=east] at (0.1, 1.7) {TD};
    \draw[black!50, thick] (0.2, 1.6) -- (0.7, 1.7) -- (1.2, 1.5) -- (1.5, 1.65);
    \fill[black] (1.5, 1.65) circle (2pt);
    \node[annot] at (1.5, 1.35) {$x_t$};
    \draw[tdcolor, thick] (1.5, 1.65) -- (1.9, 2.0) -- (2.3, 1.7) -- (2.7, 2.1) -- (3.2, 1.85);
    \draw[tdcolor, thick, dashed] (1.5, 1.65) -- (1.9, 1.3) -- (2.3, 1.6) -- (2.7, 1.2) -- (3.2, 1.35);
    \draw[<->, tdcolor, very thick] (3.25, 1.35) -- (3.25, 1.85);
    \node[annot, tdcolor] at (2.4, 2.25) {diff.\ $\mathbf{u}$, high variance};

    \node[rowlabel, anchor=east] at (0.1, 0.6) {$\phi$};
    \draw[black!50, thick] (0.2, 0.6) -- (0.7, 0.7) -- (1.2, 0.5) -- (1.5, 0.65);
    \fill[black] (1.5, 0.65) circle (2pt);
    \node[annot] at (1.5, 0.35) {$x_t$};
    \draw[phicolor, thick] (1.5, 0.65) -- (1.9, 0.9) -- (2.3, 0.7) -- (2.7, 0.95) -- (3.2, 0.8);
    \draw[phicolor, thick, dashed] (1.5, 0.65) -- (1.9, 0.75) -- (2.3, 0.55) -- (2.7, 0.80) -- (3.2, 0.65);
    \draw[<->, phicolor, very thick] (3.25, 0.65) -- (3.25, 0.8);
    \node[annot, phicolor] at (2.4, 0.1) {same $\mathbf{u}$, low variance};

    \node[annot, text width=4cm, align=center] at (\pw/2-0.3, -0.35)
        {Shared $\mathbf{u}$ cancels noise};
\end{scope}

\begin{scope}[shift={(2*\pw + 2*\gap, 0)}]
    \fill[black!3, rounded corners=3pt] (-0.7, -0.7) rectangle (\pw, \ph+0.4);
    \node[paneltitle] at (\pw/2-0.2, \ph+0.15) {(c) Delayed Reward};

    \node[rowlabel, anchor=east] at (-0.1, 1.7) {TD};
    \foreach \i in {1,...,4} {
        \node[draw, circle, inner sep=2pt, fill=white] (s\i) at ({0.2 + (\i-1)*0.8}, 1.7) {\tiny $s_{\i}$};
    }
    \node[draw, circle, inner sep=2pt, fill=yellow!50] (sY) at (3.5, 1.7) {\tiny $Y$};
    \draw[->, tdcolor, thick] (sY.west) -- (s4.east);
    \draw[->, tdcolor, thick, dashed] (s4.west) -- (s3.east);
    \draw[->, tdcolor, thick, densely dotted] (s3.west) -- (s2.east);
    \node[annot, tdcolor] at (3.1, 2.0) {ep1};
    \node[annot, tdcolor] at (2.3, 2.0) {ep2};
    \node[annot, tdcolor] at (1.5, 2.0) {ep3};
    \node[annot, tdcolor, anchor=west] at (3.3, 2.2) {$O(T)$};

    \node[rowlabel, anchor=east] at (-0.1, 0.55) {$\phi$};
    \foreach \i in {1,...,4} {
        \node[draw, circle, inner sep=2pt, fill=white] (p\i) at ({0.2 + (\i-1)*0.8}, 0.55) {\tiny $s_{\i}$};
    }
    \node[draw, circle, inner sep=2pt, fill=yellow!50] (pY) at (3.5, 0.55) {\tiny $Y$};
    \draw[->, phicolor, thick] (pY.north) to[out=120, in=60] (p4.north);
    \draw[->, phicolor, thick] (pY.north) to[out=130, in=50] (p3.north);
    \draw[->, phicolor, thick] (pY.north) to[out=140, in=40] (p2.north);
    \draw[->, phicolor, thick] (pY.north) to[out=150, in=30] (p1.north);
    \node[annot, phicolor, anchor=west] at (3.3, 1.0) {$O(1)$};

    \node[annot, text width=4cm, align=center] at (\pw/2-0.2, -0.35)
        {Counterfactual sim $\Rightarrow$ instant credit};
\end{scope}

\end{tikzpicture}
\caption{Three dimensions where $\phi$-redistribution reduces variance.
\emph{(a) Sparse causality.} TD updates all $T$ actions; $\phi$ updates only the $k'$ causal actions, reducing variance by factor $k'/T$.
\emph{(b) High stochasticity.} TD compares returns under different noise realizations (high variance); $\phi$ uses shared exogenous noise $\mathbf{u}$, so the noise cancels in the difference.
\emph{(c) Delayed reward.} TD(0) propagates credit one step per episode, requiring $O(T)$ episodes; $\phi$ assigns credit to all steps in one episode via counterfactual simulation.}
\vspace{-0.15in}
\label{fig:three-dimensions}
\end{figure*}
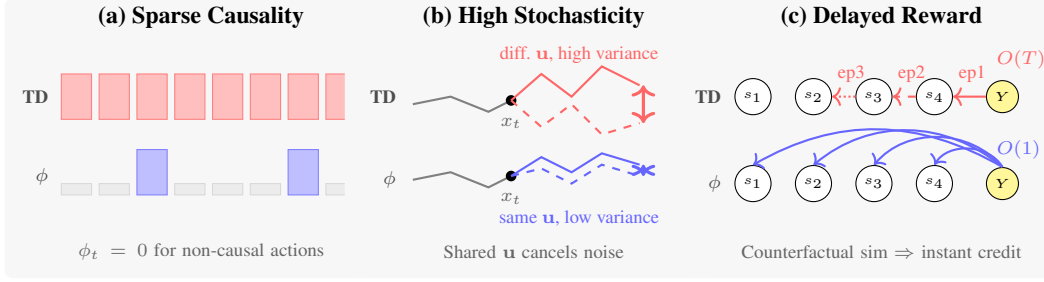

%% file: sections/4-estimating.tex
\section{Estimating Counterfactual credit for Policy Optimization}
\label{sec:computing}

This section develops efficient estimators for counterfactual Shapley values with $O(T M_{\delta,\epsilon})$ total complexity, where $M_{\delta,\epsilon} = \sigma^2/((1-\gamma)^2 \delta \epsilon^2)$, and integrates them into  policy optimization as $\phi$-PPO.

\paragraph{Importance Sampling Equivalence.}

Each coalition evaluation $f(\mathbf{z})$ requires $O(T)$ simulation time (Algo.~\ref{alg:ctf-sim}).
Computing $T$ Shapley values independently would require $M$ samples per value, yielding $O(T^2 M)$ total cost.
The kernel formulation (Eq.~\ref{eqn:kernel}) amortizes this: a single coalition evaluation updates all $T$ values via different kernel weights $\kappa_t(\mathbf{z})$, reducing cost to $O(T M)$.

\begin{definition}[Coalition Estimator]
\label{def:coalition-estimator}
Given samples $\mathbf{z}^{(m)} \sim Q$ and NTE $f(\mathbf{z}) = Y - Y^{\mathbf{z}}$, where $Y^{\mathbf{z}} \equiv Y_{\sigma_{\mathbf{z}}}$ is the counterfactual return, the \emph{coalition estimator} is
$
    \hat{\phi}_t = \frac{1}{M} \sum_{m=1}^M \frac{\kappa_t(\mathbf{z}^{(m)})}{Q(\mathbf{z}^{(m)})} \cdot f(\mathbf{z}^{(m)}).
$
\end{definition}

Under a mild assumption, we can derive the variance-reduction optimal proposal distribution $Q^*(\mathbf{z})$ for coalition sampling.

\begin{restatable}[Optimal Proposal]{theorem}{ThmOptimalProposal}
\label{thm:optimal-proposal}
Under Assumption~\ref{assum:second-moment}, the variance-minimizing proposal is $Q^*(\mathbf{z}) = q^*_k / \binom{T}{k}$ where $k = |\mathbf{z}|$ and $q^*_k \propto \sqrt{c_k}$ with:
\begin{align}
    c_k = \begin{cases}
        \frac{1}{T}\bigl(\frac{1}{k} + \frac{1}{T-k}\bigr) & k \in \{1,\ldots,T-1\} \\[4pt]
        \frac{1}{T^2} & k = T
    \end{cases}
\end{align}
Under $Q^*$, the estimator achieves $(\epsilon, \delta)$-accuracy with $M = O(M_{\delta,\epsilon})$ samples.
\end{restatable}

The optimal $q^*_k \propto \sqrt{c_k}$ concentrates on extreme coalition sizes ($k \approx 1$ or $k \approx T-1$) where kernel weights $|\kappa_t|$ are largest.
In contrast to uniform sampling, which induces variance exponential in $T$ (extreme sizes have probability $O(2^{-T})$), the optimal proposal achieves $O(1)$ second moment.
Variance is low because $Y$ and $Y^{\mathbf{z}}$ share exogenous noise via common random numbers~\citep{glasserman1992commonrandomnumbers}: the difference $f(\mathbf{z}) = Y - Y^{\mathbf{z}}$ isolates the effect of action choices while canceling environmental stochasticity. In practice, we also use antithetic sampling~\citep{covert2021improving}: each coalition $\mathbf{z}$ is paired with its complement $\mathbf{1} - \mathbf{z}$ to further reduce variance.



\paragraph{Bootstrapped Estimation.} 
The basic estimator Def.~\eqref{def:coalition-estimator} requires full trajectory simulation for each coalition.
Bootstrapping with a learned value function $V \approx V_\pi$ reduces variance by replacing future counterfactual rewards with value estimates. We use TD-$\lambda$~\citep[Sec. 12.2]{sutton2018} that interpolates between Monte Carlo and bootstrapped estimates:
\begin{align}
    G^{\mathbf{z},\lambda}_t = y^{\mathbf{z}}_t + \gamma[(1-\lambda) V(s^{\mathbf{z}}_{t+1}) + \lambda G^{\mathbf{z},\lambda}_{t+1}], \label{eq:lambda-return}
\end{align}
where $\lambda = 1$ yields pure Monte Carlo ($G^{\mathbf{z},1}_t = Y^{\mathbf{z}}_t$) and $\lambda = 0$ yields one-step TD ($G^{\mathbf{z},0}_t = y^{\mathbf{z}}_t + \gamma V(s^{\mathbf{z}}_{t+1})$). Then we define $\lambda$-bootstrapped NTE to decompose $Y^{\mathbf{z}}$ into exact rewards before step $t$ and $\lambda$-returns from $t$ onwards.
Let $R^{\mathbf{z}}_t = \sum_{s<t} \gamma^{s-1} y^{\mathbf{z}}_s$ denote the cumulative counterfactual reward before step $t$.
The per-step treatment effect is
\begin{align}
    f^{\mathbf{z}}_t = Y - (R^{\mathbf{z}}_t + \gamma^{t-1} G^{\mathbf{z},\lambda}_t). \label{eq:lambda-nte}
\end{align}
When $\lambda = 1$, we have $R^{\mathbf{z}}_t + \gamma^{t-1} G^{\mathbf{z},1}_t = Y^{\mathbf{z}}$, so $f^{\mathbf{z}}_t = f(\mathbf{z})$ recovers the exact NTE.
For $\lambda < 1$, the identity $\E[G^{\mathbf{z},\lambda}_t \mid s^{\mathbf{z}}_t] = V_\pi(s^{\mathbf{z}}_t) + O(\epsilon_V)$ implies $\E[f^{\mathbf{z}}_t] = f(\mathbf{z}) + O(\epsilon_V)$ where $\epsilon_V = \|V - V_\pi\|_\infty$.

\begin{definition}[$\lambda$-Bootstrapped Estimator]
\label{def:lambda-estimator}
Given the per-step treatment effect $f^{\mathbf{z}}_t$ Eq.~\eqref{eq:lambda-nte}, the \emph{$\lambda$-bootstrapped estimator} is
$
    \hat{\phi}^\lambda_t = \frac{1}{M} \sum_{m=1}^M \frac{\kappa_t(\mathbf{z}^{(m)})}{Q^*(\mathbf{z}^{(m)})} \cdot f^{\mathbf{z}^{(m)}}_t. 
$
\end{definition}

By incorporating the contraction rate of $\lambda$-return~\citep{tsitsiklis1997analysis}, we derive the following bias bound for our $\phi$-value estimator.
\begin{restatable}[Bias Bound for $\hat{\phi}^\lambda$]{theorem}{ThmBootstrappedPhi}
\label{thm:lambda-phi}
Under Assumptions~\ref{assum:ctf-indep}, \ref{assum:second-moment}:
$
    |\E[\hat{\phi}^\lambda_t] - \phi_t| \leq \frac{2\gamma(1-\lambda)}{1-\gamma\lambda} \|V - V_\pi\|_\infty.
$
\end{restatable}
The estimator is unbiased when $\lambda = 1$ (pure MC) or $\epsilon_V = 0$ (perfect value function).
Decreasing $\lambda$ trades bias for variance by truncating the effective horizon~\citep[Ch.~12]{sutton2018}. Algo.~\ref{alg:bootstrapped-estimator} summarizes the overall procedure for counterfactual Shapley estimation.


\begin{algorithm}[t]
\caption{\textsc{L3Est}: Counterfactual Shapley Estimation}
\label{alg:bootstrapped-estimator}
\begin{algorithmic}[1]
\REQUIRE Trajectory $\tau$ of length $T$; value network $V_\psi$; coalition samples $M$
\ENSURE Attributions $\hat{\phi}_{1:T}$, TD errors $\delta_{1:T}$, $\delta^{\mathbf{z}}_{1:T,1:M}$
\FOR{$m = 1$ to $M$}
    \STATE Sample coalition $\mathbf{z}^{(m)} \sim Q^*$ for $T$ players \COMMENT{Thm.~\ref{thm:optimal-proposal}}
    \STATE Simulate counterfactual $\tau^{\mathbf{z}^{(m)}}$ via \textsc{CtfSim} \COMMENT{Algo.~\ref{alg:ctf-sim}}
    \FOR{$t = 1$ to $T$}
        \STATE Compute $\delta^{\mathbf{z}^{(m)}}_{t} \gets G^{\mathbf{z}^{(m)},\lambda}_t - V(s^{\mathbf{z}^{(m)}}_t)$
    \ENDFOR
\ENDFOR
\STATE Compute TD-error $\delta_{1:T}$ from observed trajectory 
\STATE Compute $\hat{\phi}_{1:T}$ by averaging over $M$ samples \COMMENT{Def.~\ref{def:lambda-estimator}}
\STATE \RETURN $\hat{\phi}_{1:T}, \delta_{1:T}, \delta^{\mathbf{z}}_{1:T,1:M}$
\end{algorithmic}
\end{algorithm}

\paragraph{$\phi$-PPO.}
$\phi$-PPO integrates Shapley credit assignment into PPO~\citep{schulman2017ppo} via two mechanisms.
First, the $\phi$-MDP $\cM^\phi_\pi$ (Definition~\ref{def:phi-mdp}) reformulates credit as per-step reward $r^\phi_\pi(s,x) = \E[\phi_t \mid S_t = s, X_t = x]$, enabling standard actor-critic methods.
Second, Prioritized Trajectory Replay (PTR) focuses updates on high-$|\phi|$ subtrajectories rather than individual transitions as in PER~\citep{DBLP:journals/corr/SchaulQAS15prioritizedexpreplay}.
See Algo.~\ref{alg:phi-ppo} and Sec.~\ref{sec:phi-ppo} in appendix for the details.
In each iteration, $\phi$-PPO samples sub-trajectories via PTR, computes $\phi$-estimates, and updates both values and policy networks.
Advantages are normalized by batch standard deviation (not z-scored, to preserve $\phi=0$ as the no-effect baseline).

We admit that Algo.~\ref{alg:bootstrapped-estimator} requires $O(T  M)$ value network evaluations per trajectory (one per counterfactual state per coalition sample) while standard PPO requires $O(T)$ evaluations.
But this overhead is acceptable when: (1) value inference is cheaper than environment simulation for complex environments; (2) the $M$ evaluations can be batched for GPU efficiency.
In practice, we use $M = 1$ coalition sample per trajectory, matching standard actor-critic overhead while still benefiting from counterfactual variance reduction.

%% file: sections/5-experiments.tex
\section{Experiments}
\label{sec:experiments}

We evaluate $\phi$-PPO on two training benchmarks: SkillLuck (Sec.~\ref{sec:skill-luck}) and Combinatorial Lock (Sec.~\ref{sec:lock}), illustrate $\phi$-value attributions on DoorKey (Sec.~\ref{sec:doorkey}), and provide a controlled ablation on the Fork MDP in Appendix~\ref{sec:fork-exp}.
Throughout, a run \emph{succeeds} iff $P(\text{optimal}) > 0.9$ at all causal states for $200$ consecutive episodes.
We compare all methods on total environment steps: $\phi$-PPO uses $(1{+}M)T$ steps per episode ($1$ observed trajectory plus $M$ counterfactual rollouts); all baselines use $T$ steps per episode.
No method receives extra simulator access (full hyperparameters in Table~\ref{tab:hyperparams}).
We compare against PPO, REINFORCE, and eight prior credit assignment methods: HCA-state, HCA-return~\citep{harutyunyan2019hindsight}, CCA~\citep{pmlr-v139-mesnard21a}, CoCoA~\citep{meulemans2023would}, QCA~\citep{mesnard2023quantile}, RUDDER~\citep{arjona2019rudder}, Synthetic Returns~\citep{raposo2021synthetic}, and H-DICE~\citep{velu2024hindsightdice}.
\begin{enumerate}[label=\textbf{Q\arabic*:}, leftmargin=23pt, topsep=0pt, parsep=0pt, itemsep=1pt]
    \item Do $\phi$-values assign credit accurately, and does this translate into a learning advantage?
    \item Can $\phi$-PPO learn the optimal policy under delayed, sparse, and highly stochastic rewards?
    \item How sample efficient is $\phi$-PPO compared to vanilla PPO and prior credit assignment methods?
\end{enumerate}

\subsection{Skill vs.\ Luck}
\label{sec:skill-luck}

The SkillLuck MDP separates skill-based from luck-based rewards to test whether correct causal attributions translate into learning (\textbf{Q1}, \textbf{Q2}, \textbf{Q3}).
The agent chooses from $\mathcal{A} = \{0, 1\}$ at each of $T{=}100$ steps ($\gamma{=}0.999$, $M{=}1$).
Every reward includes i.i.d.\ noise $\mathcal{N}(0, \sigma^2)$; we sweep $\sigma \in \{0, 3\}$ to test robustness to stochasticity.
The first action $x_1$ selects one of two branches.
Actions at intermediate steps ($2 \leq t \leq T{-}2$) do not affect the trajectory or reward.
On the \emph{skill} branch, the final action determines the reward: $x_{T-1} {=} 1$ yields reward $1$, while $x_{T-1} {=} 0$ yields $0$.
On the \emph{luck} branch, the agent receives a reward of $\operatorname{Bern}(0.5) \cdot \gamma^{T-3}$ at step $2$, regardless of any action.
The $\gamma^{T-3}$ scaling ensures both branches yield returns of comparable magnitude.
Under the optimal final action, the skill branch has expected discounted return $\gamma^{T-2}$ while the luck branch has $0.5\gamma^{T-2}$, so the optimal policy chooses skill.
On the skill branch, two actions are causal: $x_1$ and $x_{T-1}$.
On the luck branch, only $x_1$ is causal because it determines the branch; all subsequent actions have zero causal effect.

\paragraph{Predictions.}
$\phi$-values should assign zero credit to $x_{T-1}$ on luck (\textbf{D1}) and more credit to $x_1$ on skill than luck (\textbf{D4}); Figure~\ref{fig:intro example} confirms this.
HCA fails because skill and luck branches share the same terminal observation.
At $T{=}100$, methods that cannot propagate delayed credit default to the luck branch.

\begin{figure}[t]
\centering
\begin{minipage}[c]{0.40\textwidth}
\centering
\input{figures/skillluck-mdp}
\end{minipage}%
\hfill
\begin{minipage}[c]{0.55\textwidth}
\centering
\includegraphics[width=\textwidth]{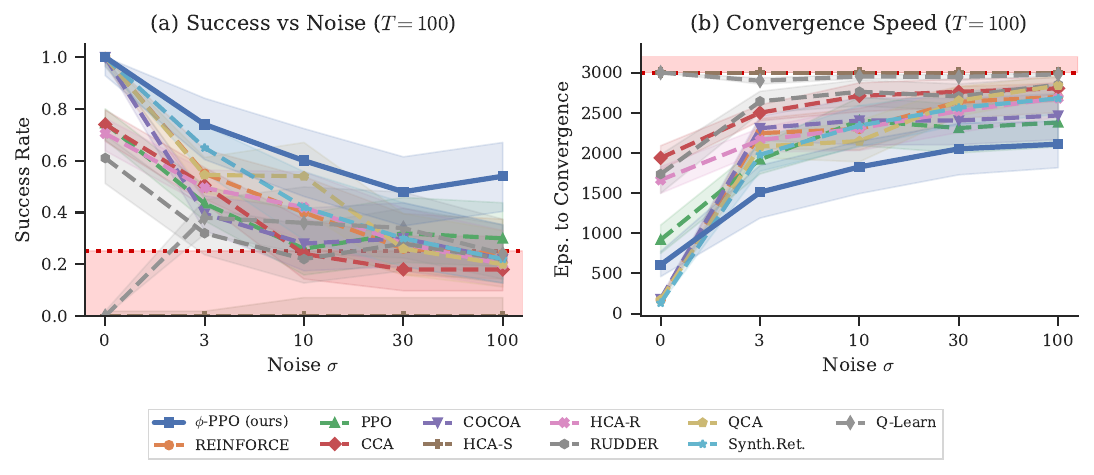}
\end{minipage}
\caption{\textbf{Skill vs.\ Luck results} ($T{=}100$, $\gamma{=}0.999$, $M{=}1$, $n{=}50$ seeds per condition).
\emph{Left:} MDP structure; the first action selects skill or luck, and only the final action on the skill branch affects the reward.
\emph{(a)} Success rate with Wilson 95\% CIs across noise levels $\sigma \in \{0, 3, 10, 30, 100\}$.
\emph{(b)} Mean episodes to convergence with bootstrap 95\% CIs; failures imputed at the budget limit ($3{,}000$ episodes).}
\label{fig:skill-luck-benchmark}
\end{figure}

\paragraph{Results (Figure~\ref{fig:skill-luck-benchmark}).}
At $\sigma{=}0$, several baselines match or exceed $\phi$-PPO because the noiseless setting poses no variance challenge; $\phi$-PPO's slight shortfall reflects $M{=}1$ approximation noise.
HCA fails as predicted.
At $\sigma{=}3$, the ranking reverses: $\phi$-PPO outperforms all baselines because shared noise cancels in the counterfactual difference (Proposition~\ref{prop:stochasticity}), while all other methods degrade.
REINFORCE and PPO both fail, consistent with predictions.
Among successful runs (right panel), $\phi$-PPO also converges fastest.

\subsection{Combinatorial Lock}
\label{sec:lock}

\begin{figure}[t]
\centering
\begin{minipage}[c]{0.40\textwidth}
\centering
\input{figures/lock-mdp}
\end{minipage}%
\hfill
\begin{minipage}[c]{0.55\textwidth}
\centering
\includegraphics[width=\textwidth]{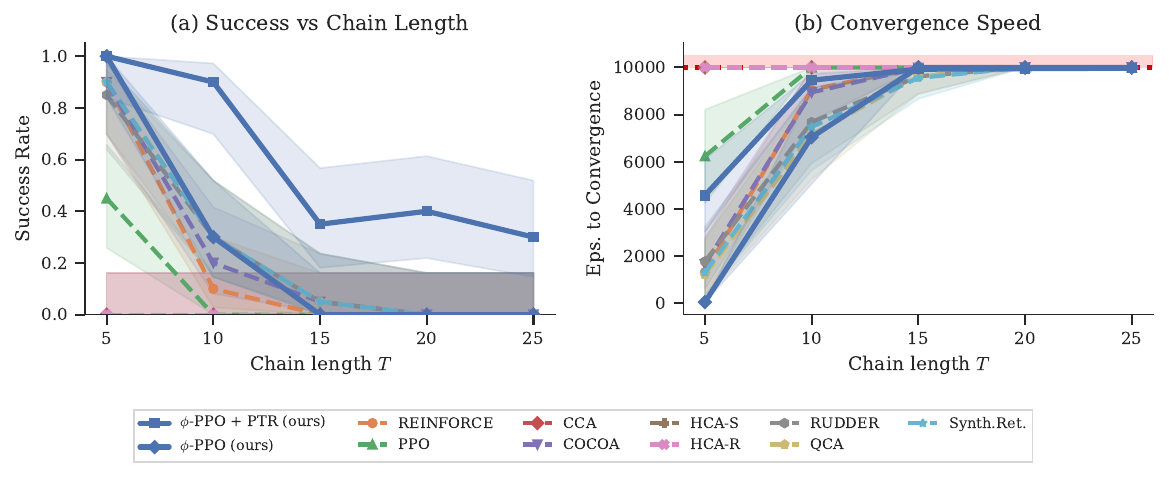}
\end{minipage}
\caption{\textbf{Combinatorial Lock results} ($T \in \{5, 10, 15, 20, 25\}$, $n{=}20$ seeds, each seeded with one optimal trajectory).
\emph{Left:} MDP structure; \textsc{Correct} advances one step (zero reward) until the final transition yields $Y{=}1$ (discounted return $\gamma^{T-1}$, twice any shortcut's $\gamma^T/2$), while \textsc{Shortcut} from position $p$ terminates with $\gamma^{T-p}/2$.
\emph{(a)} Success rate with Wilson 95\% CIs.
\emph{(b)} Mean episodes to convergence with bootstrap 95\% CIs; failures imputed at the budget limit ($10{,}000$ episodes).
Solid lines are $\phi$-PPO variants; dashed lines are baselines.}
\label{fig:lock}
\end{figure}

The Combinatorial Lock is a $T$-step chain with tempting immediate rewards that tests delayed credit assignment (\textbf{Q2}, \textbf{Q3}).
The agent chooses from $\mathcal{A} = \{0, 1\}$ at each position $p \in \{0, \ldots, T{-}1\}$: action $0$ (\textsc{Correct}) advances to $p{+}1$ with zero immediate reward; action $1$ (\textsc{Shortcut}) terminates with immediate reward $\gamma^{T-p}/2$.
Reaching position $T$ via $T$ consecutive \textsc{Correct} actions yields reward $+1$ (discounted return $\gamma^{T-1}$).
The ratio of shortcut to optimal return is $\gamma/2 < 1$ at every position, so the shortcut is always suboptimal despite positive immediate reward.
We sweep $T \in \{5, 10, 15, 20, 25\}$ with $\gamma = 0.99$ and $M{=}16$ coalition samples.
Every run begins with one optimal trajectory in the replay buffer so that all methods start from the same initial information.
This isolates \emph{exploitation} (can the method propagate credit from a known success?) from exploration, which becomes vanishingly unlikely as $T$ grows.

\paragraph{Predictions.}
Every action along the optimal path is causal ($k'/T = 1$), so the sparse causality benefit (Proposition~\ref{prop:sparse-causality}) does not apply.
Combinatorial Lock isolates the delayed reward benefit (Proposition~\ref{prop:td-propagation}):
(1)~$\phi$-PPO should propagate credit from the terminal reward to position $0$ in $O(1)$ episodes by evaluating full counterfactual returns;
(2)~TD-based methods require $O(T)$ bootstrap steps, and the shortcut temptation should lock the policy before credit reaches position $0$, causing baseline success to collapse as $T$ grows.

\paragraph{Results (Figure~\ref{fig:lock}).}
Both predictions are confirmed.
In the left panel, $\phi$-PPO + PTR (solid blue) is the only method that scales beyond $T{=}10$; every baseline (dashed) collapses by $T{=}15$.
The right panel confirms that $\phi$-PPO + PTR converges at a roughly constant episode count across chain lengths, consistent with $O(1)$ propagation, while baselines flatline at the budget ceiling.
Without PTR (solid orange), the seed trajectory is diluted by failures; PPO + PTR/PER (dashed) also fail, confirming that neither prioritized replay nor $\phi$-values alone suffice—both are necessary.
Performance degrades at large $T$ because random coalition swaps are exponentially unlikely to preserve the full optimal path, so most counterfactual rollouts return zero and estimator variance increases.

\subsection{DoorKey}
\label{sec:doorkey}

\begin{figure}[t]
  \centering
  \includegraphics[width=\linewidth]{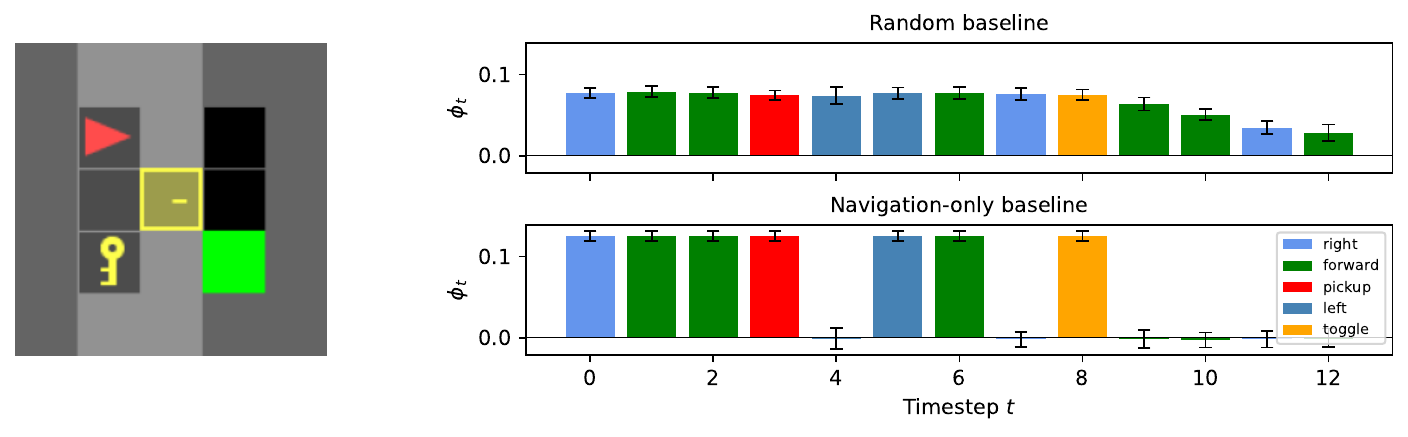}
  \caption{\textbf{DoorKey attributions} ($M{=}4096$, $\gamma{=}0.99$, 20 seeds).
  \emph{Left:} the agent (red) must collect the key, unlock the door, and reach the goal (green).
  \emph{Top right:} uniform random baseline.
  \emph{Bottom right:} navigation-only baseline that follows the shortest path ignoring walls.
  }
  \label{fig:doorkey-phi}
\end{figure}

We compute $\phi$-value attributions along an optimal trajectory in the $5{\times}5$ DoorKey environment from MiniGrid~\citep{gym_minigrid} (\textbf{Q1}).
The agent selects from 7 actions (navigation, pickup, toggle, etc.) across $|\mathcal{S}| = 400$ states ($5 {\times} 5$ grid $\times$ 4 orientations $\times$ 2 key states $\times$ 2 door states).
The optimal trajectory reaches the goal in 13 steps.

\paragraph{Baseline choice determines the causal story.}
Under a uniform random baseline (top right), $\phi_t$ is nearly uniform because DoorKey is a serial dependency chain in which every step must be correct for success.
Values taper after the door toggle ($t{=}8$) because a random agent increasingly reaches the goal by chance from nearby states.
Under a navigation-only baseline (bottom right), the causal story changes entirely.
The navigator already takes the correct action at post-door states, so those actions receive $\phi_t \approx 0$ (D1, causal admissibility).
Pickup, toggle, and first-room navigation steps that the navigator cannot replicate receive nonzero $\phi_t$ (D2, causal power), splitting the total credit equally.
$\phi$-values thus measure credit \emph{relative to the baseline policy}, not in absolute terms.
Notably, the key pickup ($t{=}3$) receives comparable $\phi_t$ to the door toggle ($t{=}8$) under both baselines, despite zero immediate reward; its value reflects the downstream effect of enabling the toggle five steps later, confirming that $\phi$-values propagate credit through dependent subgoal chains (Proposition~\ref{prop:td-propagation}).

Across all benchmarks, $\phi$-PPO converges in fewer episodes than every baseline.
At $M{=}1$ (SkillLuck), the $2{\times}$ per-episode cost is offset by ${\geq}4{\times}$ fewer episodes; at $M{=}16$ (Lock), $\phi$-PPO wins on episode efficiency but not total simulation budget, motivating future variance reduction work.

%% file: figures/skillluck-mdp.tex

\begin{tikzpicture}[
    state/.style={circle, draw, minimum size=6mm, font=\scriptsize},
    start/.style={state, fill=white},
    mid/.style={state, fill=gray!10, draw=gray!40},
    term/.style={state, fill=yellow!30, draw=orange!50},
    arr/.style={->, thick, >=stealth},
    opt/.style={arr, blue!70},
    sub/.style={arr, red!50},
    lbl/.style={font=\tiny, midway},
    ann/.style={font=\tiny, text=black!60}
]

\node[start] (s1) at (0, 0) {$s_1$};

\node[mid] (sk2) at (1.2, 1.0) {$s_2$};
\node[ann] at (2.0, 1.0) {$\cdots$};
\node[mid] (skN) at (2.8, 1.0) {$s_{T\!-\!1}$};

\node[mid] (lk2) at (1.2, -1.0) {$s_2'$};
\node[ann] at (2.0, -1.0) {$\cdots$};
\node[mid] (lkN) at (2.8, -1.0) {$s_{T\!-\!1}'$};

\node[term] (sF) at (4.8, -0.5) {$s_F$};
\node[mid]  (sW) at (4.8, 1.0) {$s_W$};

\draw[opt] (s1) -- node[lbl, above, sloped] {$x_1{=}0$} (sk2);
\draw[sub] (s1) -- node[lbl, below, sloped] {$x_1{=}1$} (lk2);

\draw[arr, gray!40] (sk2) -- (1.65, 1.0);
\draw[arr, gray!40] (2.35, 1.0) -- (skN);
\draw[arr, gray!40] (lk2) -- (1.65, -1.0);
\draw[arr, gray!40] (2.35, -1.0) -- (lkN);

\draw[opt] (skN) -- node[lbl, below left] {\tiny$x_{T\!-\!1}{=}1,\;Y{=}1$} (sF);
\draw[arr, gray!40] (skN) -- node[lbl, above] {\tiny$x_{T\!-\!1}{=}0$} (sW);

\draw[sub] (lkN) -- node[lbl, below, sloped] {\tiny$\forall x_{T\!-\!1}$} (sF);

\node[ann, blue!70, above=1pt of sk2] {\emph{skill}};
\node[ann, red!50, below=1pt of lk2] {\emph{luck}};
\node[ann, below left=1pt of lk2, red!50] {$Y{\sim}\operatorname{Bern}(\tfrac{1}{2})$};

\end{tikzpicture}

%% file: figures/lock-mdp.tex

\begin{tikzpicture}[
    state/.style={circle, draw, minimum size=5.5mm, font=\scriptsize},
    start/.style={state, fill=white},
    mid/.style={state, fill=gray!10, draw=gray!40},
    term/.style={state, fill=yellow!30, draw=orange!50},
    arr/.style={->, thick, >=stealth},
    opt/.style={arr, blue!70},
    sub/.style={arr, red!50, dashed},
    lbl/.style={font=\tiny, midway},
    ann/.style={font=\tiny, text=black!60}
]

\node[start] (s0) at (0, 0) {$0$};
\node[mid]   (s1) at (1.0, 0) {$1$};
\node[ann]        at (1.7, 0) {$\cdots$};
\node[mid]  (sT1) at (2.4, 0) {};
\node[term]  (sT) at (3.4, 0) {$T$};

\draw[opt] (s0) -- (s1);
\draw[arr, blue!70] (s1) -- (1.35, 0);
\draw[arr, blue!70] (2.05, 0) -- (sT1);
\draw[opt] (sT1) -- node[lbl, above] {$Y{=}1$} (sT);

\draw[sub] (s0) to[out=-60, in=-170] node[lbl, below, pos=0.3] {$\frac{\gamma^T}{2}$} (sT);
\draw[sub] (s1) to[out=-55, in=-150] node[lbl, below, pos=0.4] {$\frac{\gamma^{T\!-\!1}}{2}$} (sT);

\node[ann, blue!70, above=1pt of s0, xshift=5pt] {correct ($Y{=}0$)};
\node[ann, red!50] at (1.7, -1.3) {shortcut};

\end{tikzpicture}

%% file: sections/6-conclusion.tex
\section{Conclusion}
\label{sec:conclusion}


We introduced Counterfactual Shapley Credit Assignment, a principled framework that addresses the fundamental temporal credit assignment problem by grounding reward attribution in causal theory via Structural Causal Models. By utilizing the Counterfactual Shapley Value ($\phi$-value), we isolate an agent’s true causal contribution from environmental stochasticity and spurious correlations, inducing a dual $\phi$-MDP that preserves the original optimal policy while providing significantly clearer learning signals. Our framework specifically addresses the challenges of high stochasticity, sparse causality, and delayed rewards through a computationally efficient linear-time estimator and the resulting $\phi$-PPO algorithm. Empirical results across challenging benchmarks demonstrate that $\phi$-PPO achieves superior sample efficiency and convergence compared to existing baselines, offering a scalable and theoretically sound path toward more efficient and explainable reinforcement learning.

%% file: sections/E-full-pseudocode.tex
\section{Counterfactual Simulation}
\label{sec:full code}

\begin{algorithm}[ht]
\caption{\textsc{CtfSim}: Counterfactual Trajectory Simulation}
\label{alg:ctf-sim}
\begin{algorithmic}[1]
\REQUIRE Coalition $\mathbf{z} \in \{0,1\}^T$, trajectory $\tau = (s_{1:T}, x_{1:T}, y_{1:T})$, policies $\pi$, $\pi_{\mathrm{base}}$
\ENSURE $(s^{\mathbf{z}}_{1:T}, y^{\mathbf{z}}_{1:T}, Y^{\mathbf{z}})$: counterfactual states, rewards, and total return
\STATE $s^{\mathbf{z}}_1 \gets s_1$
\FOR{$t = 1$ to $T$}
    \IF{$z_t = 0$ \AND $s^{\mathbf{z}}_t = s_t$}
        \STATE $x^{\mathbf{z}}_t \gets x_t$ \COMMENT{On-track: reuse observed}
    \ELSIF{$z_t = 0$}
        \STATE $x^{\mathbf{z}}_t \sim \pi(\cdot \mid s^{\mathbf{z}}_t)$ \COMMENT{Off-track: sample from $\pi$}
    \ELSE
        \STATE $x^{\mathbf{z}}_t \sim \pi_{\mathrm{base}}(\cdot \mid s^{\mathbf{z}}_t)$ \COMMENT{Intervention}
    \ENDIF
    \IF{$(s^{\mathbf{z}}_t, x^{\mathbf{z}}_t) = (s_t, x_t)$}
        \STATE $(s^{\mathbf{z}}_{t+1}, y^{\mathbf{z}}_t) \gets (s_{t+1}, y_t)$ \COMMENT{Reuse observed}
    \ELSE
        \STATE $s^{\mathbf{z}}_{t+1} \sim P(\cdot \mid s^{\mathbf{z}}_t, x^{\mathbf{z}}_t)$
        \STATE $y^{\mathbf{z}}_t \sim r(s^{\mathbf{z}}_t, x^{\mathbf{z}}_t, \cdot)$
    \ENDIF
\ENDFOR
\STATE $Y^{\mathbf{z}} \gets \sum_{t=1}^T \gamma^{t-1} y^{\mathbf{z}}_t$ \COMMENT{Discounted return}
\STATE \RETURN $(s^{\mathbf{z}}_{1:T}, y^{\mathbf{z}}_{1:T}, Y^{\mathbf{z}})$
\end{algorithmic}
\end{algorithm}

%% file: sections/D-phiPPO.tex
\section{$\phi$-PPO}
\label{sec:phi-ppo}

$\phi$-PPO integrates Shapley credit assignment into PPO~\citep{schulman2017ppo} via two mechanisms.
First, the $\phi$-MDP $\cM^\phi_\pi$ (Definition~\ref{def:phi-mdp}) reformulates credit as per-step reward $r^\phi_\pi(s,x) = \E[\phi_t \mid S_t = s, X_t = x]$, enabling standard actor-critic methods.
Second, Prioritized Trajectory Replay (PTR) focuses policy updates on high-$|\phi|$ actions when credit is sparse.

\paragraph{Intuition.}
Compared to standard PPO, $\phi$-PPO makes three changes: (1) replace per-step rewards $r_t$ with Shapley values $\hat{\phi}_t$, which isolate causal contribution from return noise; (2) add a second value network $V_\psi$ for bootstrapping counterfactual simulations (Algo.~\ref{alg:bootstrapped-estimator}); the standard PPO value network $V^\phi_\omega$ remains, now estimating advantages on the $\phi$-MDP; (3) use PTR to select trajectories for replay.
Compared to standard PER~\citep{schaul2015per}, PTR prioritizes by $|\phi_t|$ (causal impact) rather than TD error $|\delta_t|$.
Consider Pong: when the ball moves away from the paddle, actions have no causal effect on the outcome, yet TD error may be high if the value network is inaccurate.
TD-error PER would wastefully prioritize these non-causal states; PTR correctly assigns low priority since $|\phi_t| \approx 0$.
Conversely, if the policy network learns slower than the value network, TD error is low but the policy remains suboptimal at causally important states.
PTR correctly prioritizes these states because $|\phi_t|$ reflects whether the action matters, not whether the value estimate is accurate.

Realized estimates $\hat{\phi}_t$ yield unbiased policy gradients when substituted for $r^\phi_\pi$.
By Theorem~\ref{thm:lambda-phi} with $\lambda=1$, $\E[\hat{\phi}_t \mid \tau] = \phi_t$, where $\tau = (s_{1:T}, x_{1:T}, y_{1:T})$ denotes the trajectory.
The law of iterated expectations then gives $\E[\hat{\phi}_t \mid s_t, x_t] = r^\phi_\pi(s_t, x_t)$.

The algorithm maintains two value networks: $V_\psi \approx V_\pi$ for bootstrapping counterfactual simulations in Algo.~\ref{alg:bootstrapped-estimator}, and $V^\phi_\omega \approx V^\phi_\pi$ for computing $\phi$-advantages.

\paragraph{$\phi$-Advantages.}
Let $\hat{\phi}_t$ denote the Shapley estimate from Algo.~\ref{alg:bootstrapped-estimator}.
The \emph{$\phi$-TD error} measures one-step prediction error using $\hat{\phi}_t$ as reward:
\begin{align}
    \delta_t^\phi = \hat{\phi}_t + V^\phi_\omega(s_{t+1}) - V^\phi_\omega(s_t). \label{eq:phi-td}
\end{align}
The \emph{$\phi$-advantage} generalizes GAE~\citep{schulman2015high} to the $\phi$-MDP:
\begin{align}
    A_t^{\phi} = \sum_{\ell=0}^{T-t} (\lambda^\phi)^\ell \delta_{t+\ell}^\phi, \label{eq:phi-gae}
\end{align}
where $\lambda^\phi \in [0,1]$ controls the bias-variance tradeoff (distinct from $\lambda$ in Algo.~\ref{alg:bootstrapped-estimator}, which controls bootstrapping depth).
The $\phi$-return is
\begin{align}
    G_t^\phi = \sum_{\ell=0}^{T-t} \hat{\phi}_{t+\ell}, \label{eq:phi-return}
\end{align}
which serves as the training target for $V^\phi_\omega$ (Eq.~\ref{eq:vphi-loss}).
Shapley efficiency ensures $\sum_t \phi_t = Y - Y_{\emptyset}$, where $Y_{\emptyset}$ is the baseline return (all-default actions); the sum of per-step credit equals the total reward improvement.

\paragraph{Adaptive fresh/replay ratio.}
Each training iteration mixes fresh rollouts with replay from the priority buffer.
A static fresh probability $\xi$ is suboptimal: early in training, fresh rollouts discover high-$|\phi|$ states; later, replay exploits known high-impact actions.
We adapt $\xi$ via Thompson sampling \citep{russo2018tutorial}.

Let $p_j = \max_t p_{j,t}$ denote the trajectory-level priority, where $p_{j,t}$ is the priority from Eq.~\eqref{eq:ptr-priority} below.
A fresh trajectory ``wins'' if $p_j > \bar{p}_{\mathcal{B}}$.
Here $\bar{p}_{\mathcal{B}} = \sum_{i \in \mathcal{B}} p_i^2 / \sum_{i \in \mathcal{B}} p_i$ is the expected priority under proportional sampling.
Let $u \in [0,1]$ denote the fraction of fresh trajectories that win.
We track $(\alpha, \beta)$, initialized to $(1, 1)$, via exponential moving averages with decay $\gamma_\xi < 1$:
\begin{align}
    \alpha \gets \gamma_\xi \alpha + u, \quad \beta \gets \gamma_\xi \beta + (1 - u). \label{eq:xi-update}
\end{align}
The decay discounts old observations to adapt to non-stationarity.
At steady state $\alpha + \beta \to 1/(1 - \gamma_\xi)$, bounding the effective sample size.

At each iteration, sample $\xi \sim \mathrm{Beta}(\alpha + 1, \beta + 1)$ and collect $n_{\mathrm{fresh}} = \mathrm{clamp}(\lfloor \xi B \rfloor, 1, B{-}1)$ fresh rollouts.
Stratified sampling reduces variance in the fresh/replay ratio compared to per-trajectory Bernoulli sampling.

\paragraph{Importance-sampling corrections.}
The IS correction depends on the trajectory source.
Both fresh and replay trajectories use the PPO ratio for multi-epoch updates:
\begin{align}
    r_{j,t} = \pi_\theta(x_{j,t} \mid s_{j,t}) / \pi_{\theta_{\mathrm{old}}}(x_{j,t} \mid s_{j,t}), \label{eq:ppo-ratio}
\end{align}
where $\pi_{\theta_{\mathrm{old}}}$ is the policy snapshot at the start of the current training step.
The clipped surrogate objective~\citep{schulman2017ppo} prevents large policy updates:
\begin{align}
    \ell(r, A) = -\min\bigl(r A,\; \mathrm{clip}(r, 1{-}\epsilon_{\mathrm{clip}}, 1{+}\epsilon_{\mathrm{clip}}) A\bigr). \label{eq:ppo-clip}
\end{align}
\emph{Replay trajectories} additionally require a PTR importance-sampling weight to correct for priority-based sampling.
Let $P_j = p_j / \sum_{i \in \mathcal{B}} p_i$ denote the sampling probability for trajectory $j$ from buffer $\mathcal{B}$.
The normalized IS weight is:
\begin{align}
    w_j = (|\mathcal{B}| \cdot P_j)^{-\beta_{\mathrm{is}}} \big/ \max_{j' \in \mathcal{B}} (|\mathcal{B}| \cdot P_{j'})^{-\beta_{\mathrm{is}}}. \label{eq:ptr-weights}
\end{align}
The exponent $\beta_{\mathrm{is}}$ anneals from $0$ to $1$ over training; at $\beta_{\mathrm{is}} = 1$, the weights fully correct for non-uniform sampling.
Since $\phi$-values are \emph{recomputed} for replay trajectories using the current policy and value function, no policy importance weight is needed.

The full policy loss applies the PTR weight only to replay trajectories:
\begin{align}
    \mathcal{L}_\theta^{\mathrm{PTR}} = \frac{1}{B} \sum_{j=1}^B \tilde{w}_j \sum_{t=1}^{T} \ell(r_{j,t}, A_t^\phi), \quad \tilde{w}_j = \begin{cases} 1 & \text{if } j \text{ is fresh} \\ w_j & \text{if } j \text{ is replay} \end{cases}. \label{eq:ptr-loss}
\end{align}

\paragraph{PTR: Priority computation.}
When credit is sparse, PTR focuses updates on high-impact actions.
The priority for state $s_t$ is
\begin{align}
    p_t = (\phi^{(2)}_t)^{\alpha_{\mathrm{ptr}}/2} + \epsilon_{\mathrm{ptr}}, \label{eq:ptr-priority}
\end{align}
where $\phi^{(2)}_t$ is a bias-corrected exponential moving average~\citep{kingma2014adam} of squared Shapley estimates.
The priority $\sqrt{\phi^{(2)}}$ is an exponentially-weighted root-mean-square (RMS), capturing both mean magnitude and variance: $\sqrt{\E[\phi^2]} = \sqrt{\mu^2 + \sigma^2}$.
This prioritizes states where (i) actions have consistent causal impact (high $|\mu|$), and (ii) the action choice matters but the policy has not yet converged (high $\sigma$ with $\mu \approx 0$).
The exponent $\alpha_{\mathrm{ptr}}$ controls priority sharpness; $\epsilon_{\mathrm{ptr}}$ prevents starvation.

\emph{Priority decay.}
Priorities become stale as the policy improves.
Each iteration, all priorities decay by $p_j \gets \gamma_p p_j$, where $\gamma_p \in [0.99, 0.9999]$.
Trajectories with $p_j < \epsilon_{\mathrm{evict}}$ are evicted, implicitly bounding the buffer size to $O(\log(\epsilon_{\mathrm{evict}}) / \log(\gamma_p))$ iterations of data.

\paragraph{Training.}
Three components train jointly.
The value network $V_\psi$ for counterfactual bootstrapping trains on both actual and counterfactual TD errors from Algo.~\ref{alg:bootstrapped-estimator}:
\begin{align}
    \mathcal{L}_{V} = \textstyle\frac{1}{(M{+}1)BT}\Bigl(\sum_{j,t} \delta_{j,t}^2 + \sum_{j,t,m} (\delta^{\mathbf{z}^{(m)}}_{j,t})^2\Bigr). \label{eq:v-loss}
\end{align}
The $\phi$-critic $V^\phi_\omega$ regresses to the $\phi$-return:
\begin{align}
    \mathcal{L}_{V^\phi} = \textstyle\frac{1}{BT}\sum_{j,t} (G_{j,t}^\phi - V^\phi_\omega(s_{j,t}))^2. \label{eq:vphi-loss}
\end{align}
Following \citet{schulman2017ppo}, all three components train via a single combined loss with gradient clipping (norm $g_{\max}$):
\begin{align}
    \mathcal{L}^* = \mathcal{L}_\theta^{\mathrm{PTR}} + c_1 \mathcal{L}_{V} + c_2 \mathcal{L}_{V^\phi} - c_3 H(\pi_\theta), \label{eq:phi-ppo-loss}
\end{align}
where $c_1, c_2$ weight the value losses and $c_3$ weights the entropy bonus $H(\pi_\theta) = -\E_{s \sim d_\pi}\bigl[\sum_x \pi_\theta(x \mid s) \log \pi_\theta(x \mid s)\bigr]$~\citep{schulman2017ppo} (Table~\ref{tab:hyperparams}).

\paragraph{Algorithm.}
Algo.~\ref{alg:phi-ppo} summarizes $\phi$-PPO.
Each iteration samples subtrajectories via PTR, computes $\phi$-estimates, and updates all three networks.
Advantages are normalized by batch standard deviation (not z-scored, to preserve $\phi=0$ as the no-effect baseline).

\begin{algorithm}[t]
\caption{$\phi$-PPO}
\label{alg:phi-ppo}
\begin{algorithmic}[1]
\REQUIRE Policy $\pi_\theta$; value networks $V_\psi$, $V^\phi_\omega$; batch size $B$; coalition samples $M$; epochs $K$; decays $\gamma_\xi, \gamma_p$
\ENSURE Updated parameters $\theta, \psi, \omega$
\STATE Initialize priority buffer $\mathcal{B} \gets \emptyset$; pseudo-counts $\alpha, \beta \gets 0$
\FOR{iteration $= 1, 2, \ldots$}
    \STATE Sample $\xi \sim \mathrm{Beta}(\alpha + 1, \beta + 1)$; set $n_{\mathrm{fresh}} \gets \mathrm{clamp}(\lfloor \xi B \rfloor, 1, B{-}1)$
    \STATE $\forall j \in \mathcal{B}$: $p_j \gets \gamma_p p_j$; evict if $p_j < \epsilon_{\mathrm{evict}}$ \COMMENT{priority decay}
    \FOR{$j = 1$ to $B$}
        \IF{$\mathcal{B} = \emptyset$ \OR $j \le n_{\mathrm{fresh}}$}
            \STATE Roll out fresh $\tau \sim \pi_\theta$; set $\tilde{w}_j \gets 1$
        \ELSE
            \STATE Sample $\tau \propto p_\tau$ from $\mathcal{B}$; set $\tilde{w}_j$ via Eq.~\ref{eq:ptr-weights}
        \ENDIF
        \STATE Compute $\hat{\phi}_{1:T}, \delta_{1:T}, \delta^{\mathbf{z}}_{1:T,1:M} \gets \textsc{L3Est}(\tau, M)$ \COMMENT{Algo.~\ref{alg:bootstrapped-estimator}}
        \STATE Compute $\phi$-advantages $A^\phi_{1:T}$ \COMMENT{Eq.~\ref{eq:phi-gae}}
        \STATE $A^\phi_{1:T} \gets A^\phi_{1:T} / \mathrm{std}(A^\phi)$ \COMMENT{normalize, preserve $\phi{=}0$}
        \STATE Compute priorities $p^\tau_{1:T}$; update $\mathcal{B}$ \COMMENT{Eq.~\ref{eq:ptr-priority}}
    \ENDFOR
    \STATE Update $\alpha \gets \gamma_\xi \alpha + u$, $\beta \gets \gamma_\xi \beta + (1 - u)$ \COMMENT{$u$: fresh win rate}
    \STATE Optimize $\mathcal{L}^*$ w.r.t.\ $\theta, \psi, \omega$ for $K$ epochs \COMMENT{Eq.~\ref{eq:phi-ppo-loss}}
\ENDFOR
\end{algorithmic}
\end{algorithm}

\paragraph{Guarantees.}
$\phi$-PPO inherits theoretical properties from its components.

\emph{Gradient equivalence.}
By Lemma~\ref{lem:gradient-equiv}, $\phi$-values produce the same expected policy gradient as original returns: $\E[\sum_t \phi_t \nabla \log \pi_t] = \E[Y \sum_t \nabla \log \pi_t]$.
Thus $\phi$-PPO targets the same optimum as standard PPO.

\emph{Bootstrap bias.}
The $\lambda$-return (Eq.~\ref{eq:lambda-return}) introduces bias bounded by $O(\gamma(1{-}\lambda)\epsilon_V / (1{-}\gamma\lambda))$ where $\epsilon_V = \|V - V_\pi\|_\infty$ (Theorem~\ref{thm:lambda-phi}).
This bias vanishes as $V_\psi \to V_\pi$.

\emph{Trajectory replay.}
Because $\phi$-values are defined over trajectories (Definition~\ref{def:phi-mdp}), PTR replays entire subtrajectories rather than individual transitions.
This is sound by Corollary~\ref{cor:chunked-phi}: for any state $s$ and horizon $n$, the optimal policy in the undiscounted truncated $\phi$-MDP $\cM^\phi_n(s)$ equals the optimal policy in $\cM$.
Since $V^{\phi,*} = 0$ under self-baseline at optimality, no bootstrap is needed in the exact case; in practice, $\hat{V}^\phi \approx 0$ during training.
IS weights (Eq.~\ref{eq:ptr-weights}) correct for the prioritized sampling distribution, ensuring unbiased gradient estimates.

Hyperparameters are in Table~\ref{tab:hyperparams}; proofs in Appendix~\ref{app:estimation}.

%% file: sections/C-theorydetails.tex
\section{Theory Details}
\label{sec:theory detail}
\begin{definition}[Functional Dependence]
Given world $(\mathcal{M}, \mathbf{u})$, the functional dependence of $Y$ on $X$
under baseline $z = (x', \mathbf{z}')$ is:
\begin{align}
  c(\mathcal{M}, \mathbf{u}, X, Y, z) = Y_{\mathbf{z}'}(\mathbf{u}) - Y_{\mathbf{z}', x'}(\mathbf{u})
\end{align}
The baseline space is $\mathcal{Z} = \{(x', \mathbf{z}') : x' \in \mathcal{D}_X,
\mathbf{Z} \subseteq \mathbf{X} \setminus \{X\}, \mathbf{z}' \in \mathcal{D}_{\mathbf{Z}}\}$
equipped with probability measure $P^\mathcal{M}(\mathbf{Z}, X)$.
\end{definition}

This measures the outcome difference when $X$ changes from baseline $x'$ to its actual
value, with other variables held at $\mathbf{z}'$.
When $\mathbf{z}' = \emptyset$, this reduces to the unit-level total effect
$Y(\mathbf{u}) - Y_{x'}(\mathbf{u})$.
Functional dependence captures the counterfactual test: if changing $X$ changes $Y$,
then $X$ is responsible for $Y$.
This condition is necessary for causation across major definitions in the actual causation literature~\citep{halpern2016actual,beckers2018principled,beckers2021counterfactual}.

To formalize the condition, we introduce the \emph{why-query} $\mathrm{Why}(Y \mid \mathbf{e}; \mathbf{X})$, composed of three parts:
\begin{enumerate}[label={(\arabic*)}, leftmargin=2.2em, topsep=0pt, parsep=0pt, itemsep=1pt]
\item the event being explained, or \emph{explanandum}, $Y = y$, inducing alternate values, or \emph{foils}, $Y = y' \neq y$;
\item the \emph{evidence} $\mathbf{E} = \mathbf{e}$, the observed context used to infer causes; and
\item the \emph{explanatory variables} $\mathbf{X} \subseteq \mathbf{V}$, which prescribe the subset of variables to attribute to.\footnote{Explicitly selecting the set of explanatory variables as a subset of observed variables allows us to exclude irrelevant or unobserved causes such as ``Alice arriving'' or ``the Big Bang'', even though it is true that if Alice did not arrive, or the Big Bang did not occur, the funding outcome would be different! }
\end{enumerate}
We write $\mathrm{Why}(y \mid \mathbf{e})$ as shorthand for $\mathrm{Why}(Y{=}y \mid \mathbf{e}; \mathbf{X})$ when $\mathbf{X}$ is clear from context.

Given a why-query, an Explanatory Variable Attribution (EVA) assigns a numerical attribution to each explanatory variable in $\mathbf{X}$.

\begin{definition}[Explanatory Variable Attribution] \label{def:eva}
An explanatory variable attribution (EVA) is a function $\phi: \Omega \times \mathbb{W} \times \mathbf{X} \to \mathbb{R}$, where $\Omega$ is the space of SCMs over $\mathbf{V}$ and $\mathbb{W}$ is the space of why-queries $\mathrm{Why}(Y \mid \mathbf{e}; \mathbf{X})$ asking why $Y = y \in \mathbf{e}$, given evidence $\mathbf{E} = \mathbf{e}$, in terms of variables $\mathbf{X} \subseteq \mathbf{V} \setminus \{Y\}$. 
\end{definition}

With EVA, we are ready to formally define the Causal Credit Assignment criteria from Def.~\ref{def:causal credit assignment}.
\begin{definition}[Causal Credit Assignment, formal]
\label{def:causal credit assignment formal}
We define four desiderata as mappings $D_{1:4}: \Omega \times \mathbb{C} \times \Phi \to \{0, 1\}$,
where $\mathbb{C}$ is the space of causal measures and $\Phi$ is the space of EVAs.
Given SCM $\mathcal{M} \in \Omega$, causal measure $c \in \mathbb{C}$, and EVA $\phi \in \Phi$,
we say $D_i(\mathcal{M}, c, \phi) = 1$ when:
\begin{align}
  D_1: \quad &c = 0 \implies \phi = \mathbf{0}
    \tag{Causal Admissibility} \\
  D_2: \quad &\exists \mathbf{u}, z: c \neq 0 \implies \phi \neq \mathbf{0}
    \tag{Causal Power} \\
  D_3: \quad &c_\mathcal{M} = c_{\mathcal{M}'} \land P^\mathcal{M}_{\bar{c}}(Z) \not\equiv P^{\mathcal{M}'}_{\bar{c}}(Z)
    \land P^\mathcal{M}_{+}(z) \geq P^{\mathcal{M}'}_{+}(z) \notag \\
    &\implies \phi_\mathcal{M} > \phi_{\mathcal{M}'}
    \tag{Causal Normality} \\
  D_4: \quad &c_\mathcal{M} \geq c_{\mathcal{M}'} \land P^\mathcal{M}(Z) \equiv P^{\mathcal{M}'}(Z)
    \land P^\mathcal{M}(c_Z) \not\equiv P^{\mathcal{M}'}(c_Z) \notag \\
    &\implies \phi_\mathcal{M} > \phi_{\mathcal{M}'}
    \tag{Causal Effect Scaling}
\end{align}
Quantification: all premises universally quantified unless marked $\exists$.
Notation: $c$, $c_\mathcal{M}$ abbreviate $c(\mathcal{M}, \mathbf{u}, X, Y, z)$;
$\phi$, $\phi_\mathcal{M}$ abbreviate $\phi_X(\mathcal{M}, w)$;
$P \equiv P'$ denotes distributional equality;
$P^\mathcal{M}_{\bar{c}}(Z) \coloneqq P^\mathcal{M}(Z \mid c \neq 0)$;
$P^\mathcal{M}_{+}(z) \coloneqq \mathrm{sign}(c_\mathcal{M}) \cdot P^\mathcal{M}(z)$.
\end{definition}

\paragraph{Chunk-size Updates.} Theorem~\ref{thm:optimal-policy-equiv} establishes optimality equivalence for full-length trajectories optimization.
In practice, we chunk trajectories: each chunk begins at an arbitrary state $s$ (sampled from a replay buffer) and extends for $n$ steps with a value bootstrap at the end.
The following corollary shows that both modifications---arbitrary starting state and finite-horizon truncation---preserve the optimal policy.

\begin{restatable}[Chunked $\phi$-Learning]{corollary}{CorChunkedPhi}
\label{cor:chunked-phi}
For any state $s$ and horizon $n$, let $\cM^\phi_n(s)$ denote the undiscounted $\phi$-MDP starting from $s$ with actions $x_{1:n}$ and terminal value $V^\phi(s_n)$.
The optimal policy in $\cM^\phi_n(s)$ equals the optimal policy in $\cM$.
\end{restatable}

The corollary holds for \emph{any} state $s$, not only the initial state distribution.
This justifies experience replay: sampling $(s_t, x_t)$ from a buffer and computing $\phi_{t:t+n}$ yields correct policy updates regardless of trajectory origin.
Under self-baseline at optimality, $V^{\phi,*}(s) = 0$ for all $s$ (Shapley efficiency), so no bootstrap is needed in the exact case.
In practice, $\hat{V}^\phi \approx 0$ during training, and the bootstrap error vanishes as the policy converges.
This enables Prioritized Trajectory Replay (PTR) using $|\phi_t|$ as priority.
TD-error prioritization samples states where the value function is inaccurate, but accurate values do not guarantee optimal actions.
In contrast, $|\phi_t|$ measures causal contribution: large $|\phi_t|$ indicates the action significantly affected the outcome.
Policy improvement at states with large $|\phi_t|$ affects returns the most.
When credit is sparse ($k' \ll T$ causal actions), non-causal state-action pairs have $\phi_t \approx 0$ and receive minimal priority, concentrating updates on the causally-relevant subset.

\begin{table}[t]
\centering
\caption{Three dimensions where $\phi$-methods improve sample complexity.}
\label{tab:when-helps}
\smallskip
\begin{tabular}{@{}llll@{}}
\toprule
\textbf{Dimension} & \textbf{Improvement} & \textbf{Condition} & \textbf{TD Failure} \\
\midrule
Sparse causality & $\times\, k'/T$ variance & $k' \ll T$ & Wasted updates \\
High stochasticity & $\times\, (1-\rho_Y)$ variance & $\rho_Y \approx 1$ & Noise swamps signal \\
Delayed reward & $O(T) \to O(1)$ episodes & Large delay & Slow propagation \\
\bottomrule
\end{tabular}
\end{table}

%% file: sections/A-proofs.tex
\section{Proofs}
\label{app:proofs}

This appendix contains proofs and additional material omitted from the main text.

\subsection{Notation}
\label{app:notation}

\begin{table}[t]
\centering
\caption{Notation glossary}
\label{tab:notation}
\begin{tabular}{ll}
\toprule
Symbol & Meaning \\
\midrule
\multicolumn{2}{l}{\emph{SCM and MDP}} \\
$\V, \U$ & Endogenous, exogenous variables \\
$\tau$ & Trajectory: $(s_{1:T}, x_{1:T}, y_{1:T})$ \\
$S_t, X_t, Y_t$ & State, action, reward at time $t$ \\
$\pi(x \mid s)$ & Policy (action distribution given state) \\
$\gamma \in (0,1)$ & Discount factor \\
$T$ & Horizon (episode length) \\
\midrule
\multicolumn{2}{l}{\emph{Coalitions and Counterfactuals}} \\
$\mathbf{z} \in \{0,1\}^T$ & Coalition mask ($z_t = 1 \Leftrightarrow X_t$ intervened) \\
$|\mathbf{z}| = k$ & Coalition size \\
$s^{\mathbf{z}}, Y^{\mathbf{z}}$ & Counterfactual state, outcome under mask $\mathbf{z}$ \\
\midrule
\multicolumn{2}{l}{\emph{Shapley Values and Credit}} \\
$\phi_t$ & Counterfactual Shapley value of action $X_t$ \\
$\kappa_t(\mathbf{z})$ & Shapley kernel weight \\
$G_t$ & Observed return from time $t$: $\sum_{k \geq 0} \gamma^k y_{t+k}$ \\
$G_t^\phi$ & $\phi$-return from time $t$: $\sum_{\tau \geq t} \gamma^{\tau-t} \phi_\tau$ \\
$\cM^{\phi}_\pi$ & $\phi$-MDP (reward-redistributed MDP) \\
$r^\phi_\pi(s,x)$ & $\phi$-MDP reward: $\E[\phi_t \mid S_t = s, X_t = x]$ \\
\midrule
\multicolumn{2}{l}{\emph{$\phi$-PPO (Algo.~\ref{alg:phi-ppo})}} \\
$V_\psi$ & Value network for original MDP (bootstrapping) \\
$V^\phi_\omega$ & Value network for $\phi$-MDP (advantages) \\
$B$ & Batch size (number of trajectories) \\
$w_j$ & PTR importance-sampling weight for sample $j$ \\
\midrule
\multicolumn{2}{l}{\emph{Analysis}} \\
$(\epsilon, \delta)$ & Accuracy parameters \\
$M_{\delta,\epsilon}$ & Sample complexity: $\sigma^2/((1-\gamma)^2 \delta \epsilon^2)$ \\
$k'$ & Number of causal actions (with $\phi_t \neq 0$) \\
$\lambda$ & NTE mixing parameter ($0$: bootstrap, $1$: Monte Carlo) \\
$\lambda^\phi$ & GAE mixing for $\phi$-advantage estimation \\
$\rho_Y$ & Outcome correlation under shared exogenous noise \\
\bottomrule
\end{tabular}
\end{table}

\subsection{Problem Setting (Section~\ref{sec:credit-assignment})}

\begin{lemma}[Variance of Discounted Sum]
\label{lem:discounted-var}
If $\E[Y_t^2] \leq \sigma^2$ for all $t$, then $\Var[Y] \leq \sigma^2/(1-\gamma)^2$ where $Y = \sum_{t=1}^T \gamma^{t-1} Y_t$.
\end{lemma}
\begin{proof}
By Cauchy-Schwarz and $\mathrm{Var}[Y_t] \leq \E[Y_t^2] \leq \sigma^2$, we have $|\mathrm{Cov}[Y_t, Y_s]| \leq \sigma^2$.
Thus:
\begin{align}
    \mathrm{Var}[Y] = \sum_{t,s} \gamma^{t+s-2} \mathrm{Cov}[Y_t, Y_s] \leq \sigma^2 \left(\sum_{t=1}^\infty \gamma^{t-1}\right)^2 = \frac{\sigma^2}{(1-\gamma)^2}.
\end{align}
\end{proof}

\begin{proposition}[Sample Complexity is $\gamma$-Independent]
\label{prop:gamma-independent}
For relative error targets $\epsilon_{\mathrm{rel}} = \epsilon / |\phi|$, $M_{\delta,\epsilon_{\mathrm{rel}}} = O(\sigma^2/(\delta \epsilon_{\mathrm{rel}}^2 R_{\max}^2))$.
\end{proposition}
\begin{proof}
Returns scale as $O(R_{\max}/(1-\gamma))$, so $\epsilon = O(\epsilon_{\mathrm{rel}} R_{\max} / (1-\gamma))$.
Substituting: $M_{\delta,\epsilon} = \sigma^2 / ((1-\gamma)^2 \delta \epsilon^2) = \sigma^2 / (\delta \epsilon_{\mathrm{rel}}^2 R_{\max}^2)$.
\end{proof}

\subsection{Credit Assignment (Section~\ref{sec:credit-assignment})}

\ThmPhiIsCausal*
\begin{proof}
We show $\phi$-values (Shapley values of the NTE game $f(\mathbf{z}) = Y - Y_{\sigma_{\mathbf{z}}}$) satisfy each desideratum (formal definitions in Appendix~\ref{sec:theory detail}).

\textbf{D1 (Causal Admissibility).}
Suppose $c = 0$: the functional dependence of $Y$ on $X_t$ is zero for all contexts $(\mathbf{u}, z)$.
Then changing $X_t$'s intervention status does not affect the counterfactual outcome: $Y^{\mathbf{z} \cup \{t\}} = Y^{\mathbf{z}}$ for every coalition $\mathbf{z}$.
All marginal contributions vanish: $f(\mathbf{z} \cup \{t\}) - f(\mathbf{z}) = Y^{\mathbf{z}} - Y^{\mathbf{z} \cup \{t\}} = 0$.
By the Shapley null player axiom, $\phi_t = 0$.

\textbf{D2 (Causal Power).}
Suppose $\exists\, \mathbf{u}, z$ such that $c(\mathcal{M}, \mathbf{u}, X_t, Y, z) \neq 0$.
Then there exists a coalition $\mathbf{z}^*$ (corresponding to context $z$) where the marginal contribution is nonzero: $f(\mathbf{z}^* \cup \{t\}) - f(\mathbf{z}^*) \neq 0$.
Since all Shapley weights $w(\mathbf{z}) = \frac{|\mathbf{z}|!(T - |\mathbf{z}| - 1)!}{T!} > 0$, the weighted sum $\phi_t = \sum_{\mathbf{z}: z_t = 0} w(\mathbf{z})[f(\mathbf{z} \cup \{t\}) - f(\mathbf{z})]$ includes a nonzero term with a strictly positive weight.
Therefore $\phi_t \neq 0$.

\textbf{D3 (Causal Normality).}
Consider two SCMs $\mathcal{M}, \mathcal{M}'$ with identical causal effects ($c_\mathcal{M} = c_{\mathcal{M}'}$) but different baseline distributions satisfying $P^\mathcal{M}_{+}(z) \geq P^{\mathcal{M}'}_{+}(z)$.
The NTE game value $f(\mathbf{z})$ depends on the SCM through the counterfactual simulation, which samples baseline actions from $\pi_{\mathrm{base}}$.
When $P^\mathcal{M}_{+}(z) \geq P^{\mathcal{M}'}_{+}(z)$, the baseline under $\mathcal{M}$ assigns weakly higher probability to baselines that produce positive causal effects.
Each marginal contribution $f(\mathbf{z} \cup \{t\}) - f(\mathbf{z})$ measures the additional effect of intervening at step $t$.
Under $\mathcal{M}$, the ``default'' behavior (baseline) is more aligned with positive outcomes, so the departure from baseline caused by $X_t$ receives greater credit.
Formally, the Shapley value is the expected marginal contribution over uniformly random orderings; since each marginal contribution under $\mathcal{M}$ weakly exceeds that under $\mathcal{M}'$ (by the monotonicity of the NTE in baseline probabilities when the causal effect sign is fixed), we obtain $\phi_\mathcal{M} > \phi_{\mathcal{M}'}$.

\textbf{D4 (Causal Effect Scaling).}
Consider two SCMs with identical baseline distributions ($P^\mathcal{M}(Z) \equiv P^{\mathcal{M}'}(Z)$) but $c_\mathcal{M} \geq c_{\mathcal{M}'}$ (larger causal effect in $\mathcal{M}$).
Larger causal effects directly increase the magnitude of counterfactual outcome differences: for each coalition $\mathbf{z}$, $|Y^{\mathbf{z}} - Y^{\mathbf{z} \cup \{t\}}|$ is weakly larger under $\mathcal{M}$.
Since baseline distributions are identical, the Shapley weights are the same in both games.
By the Shapley monotonicity property (if player $t$'s marginal contribution in game $v$ weakly exceeds that in game $v'$ for every coalition, then $\phi_t(v) \geq \phi_t(v')$), we obtain $\phi_\mathcal{M} > \phi_{\mathcal{M}'}$.
\end{proof}

\LemGradientEquiv*
\begin{proof}
By Shapley efficiency, $\sum_t \phi_t = Y - Y^{\mathbf{1}}$ where $Y^{\mathbf{1}} = Y_{\sigma_{\mathbf{1}}}$ is the full-baseline outcome.
We show $\E[Y^{\mathbf{1}} \cdot \sum_t \nabla \log \pi_t] = 0$.

Under any baseline policy $\pi_{\mathrm{base}}$ (including self-baseline), the counterfactual $Y^{\mathbf{1}}$ replaces all observed actions with fresh samples $X^{\mathbf{1}}_{1:T} \sim \pi_{\mathrm{base}}(\cdot \mid S^{\mathbf{1}}_{1:T})$.
These samples are drawn independently of the observed actions $X_{1:T}$.
We claim $Y^{\mathbf{1}} \perp X_t \mid S_t$.
By the Markov property of the MDP, $X_t \sim \pi(\cdot \mid S_t)$ is conditionally independent of all preceding random variables given $S_t$.
The baseline outcome $Y^{\mathbf{1}}$ is a deterministic function of the shared initial state $S_1$ and the counterfactual noise (actions from $\pi_{\mathrm{base}}$ and resampled transitions), all of which are independent of $X_t$.
Although $Y^{\mathbf{1}}$ depends on $S_1$, which is an ancestor of $S_t$, the Markov property ensures $S_1 \perp X_t \mid S_t$.
Therefore $Y^{\mathbf{1}} \perp X_t \mid S_t$ for any baseline policy.
By the score function identity:
\begin{align}
    \E[Y^{\mathbf{1}} \cdot \nabla \log \pi(X_t \mid S_t)]
    = \E_{S_t}[\E[Y^{\mathbf{1}} \mid S_t] \cdot \E_{X_t}[\nabla \log \pi(X_t \mid S_t)]]
    = 0.
\end{align}
\end{proof}

\ThmOptimalPolicyEquiv*
\begin{proof}
By Lemma~\ref{lem:gradient-equiv}, $\nabla_\theta J_\phi(\theta) = \nabla_\theta J(\theta)$ for all $\theta$.
Two functions with identical gradients everywhere have identical stationary points (where the gradient vanishes).
Therefore both objectives share the same critical points, including maxima.

Explicitly, by Shapley efficiency, $J_\phi(\theta) = J(\theta) - \E_{\pi_\theta}[Y^{\mathbf{1}}]$.
The term $\E_{\pi_\theta}[Y^{\mathbf{1}}]$ (expected baseline return under self-baseline) has zero gradient by the same argument as Lemma~\ref{lem:gradient-equiv}: $Y^{\mathbf{1}}$ is independent of the realized actions $X_t$, so $\E[Y^{\mathbf{1}} \nabla \log \pi] = 0$.
Thus $J_\phi$ and $J$ differ by a function with zero gradient, and $\arg\max J_\phi = \arg\max J$.
\end{proof}

\begin{corollary}[$\phi$-Value Function]
\label{cor:phi-bellman}
The $\phi$-MDP $\cM^\phi_\pi$ has a well-defined value function $V^\phi(s) = \E[\sum_{\ell \geq 0} \phi_{t+\ell} \mid S_t = s]$ satisfying the Bellman equation $V^\phi(s) = \E_{x \sim \pi}[r^\phi_\pi(s,x) + \E_{s' \sim P(\cdot \mid s,x)}[V^\phi(s')]]$.
The $\phi$-MDP is undiscounted ($\gamma_\phi = 1$): discounting is already incorporated into the Shapley values via the NTE game outcome $Y = \sum_t \gamma^{t-1} r_t$.
\end{corollary}
\begin{proof}
The $\phi$-MDP $\cM^\phi_\pi$ is a valid MDP with state space $\cS$, action space $\cA$, transition $P(s' \mid s, x)$, discount $\gamma_\phi = 1$, and reward $r^\phi_\pi(s, x) = \E[\phi_t \mid S_t = s, X_t = x]$.
Discounting is not applied to $\phi$-rewards because the Shapley game outcome $Y = \sum_t \gamma^{t-1} r_t$ already incorporates temporal discounting; applying $\gamma$ again would double-discount.
Since $\E[|\phi_t|] < \infty$ (bounded by $\E[|Y|] < \infty$ via efficiency) and episodes are finite ($T < \infty$), the undiscounted return $\sum_t \phi_t$ is integrable.
Standard MDP theory applies: value functions exist and satisfy Bellman equations.
\end{proof}

\CorChunkedPhi*
\begin{proof}
The proof proceeds in two steps, corresponding to the two truncations.

\paragraph{Step 1: Arbitrary initial state.}
Theorem~\ref{thm:optimal-policy-equiv} shows that the optimal policy in $\cM^\phi$ equals the optimal policy in $\cM$.
By the Markov property, the optimal policy $\pi^*(x \mid s)$ depends only on the current state $s$, not on the initial state distribution or the trajectory history.
Therefore, $\pi^*$ is optimal in $\cM^\phi(s)$ for any starting state $s$.
This justifies experience replay: transitions sampled from a buffer correspond to different starting states, but all yield the same optimal policy.

\paragraph{Step 2: Finite horizon with value bootstrap.}
Consider the $n$-step truncated $\phi$-MDP $\cM^\phi_n(s)$ with terminal value $V^\phi(s_n)$.
Since the $\phi$-MDP is undiscounted ($\gamma_\phi = 1$, Corollary~\ref{cor:phi-bellman}), the objective is:
\begin{align}
    J_n^\phi(s) = \E_\pi\left[\sum_{t=1}^{n} \phi_t + V^\phi(s_n) \mid S_1 = s\right].
\end{align}
Under self-baseline at the optimal policy $\pi^*$, Shapley efficiency gives $\E_{\pi^*}[\sum_t \phi_t \mid S_1 = s] = \E_{\pi^*}[Y - Y^{\mathbf{1}} \mid S_1 = s] = V^*(s) - V^*(s) = 0$ for all $s$.
Therefore $V^{\phi,*}(s) = 0$ for all $s$, and the $n$-step objective simplifies to $J_n^\phi(s) = \E_\pi[\sum_{t=1}^n \phi_t \mid S_1 = s]$.
By Lemma~\ref{lem:gradient-equiv} applied to the $n$-step game (the proof uses Assumption~\ref{assum:ctf-indep} step-by-step and carries through for any finite horizon), $\nabla_\theta \E_\pi[\sum_{t=1}^n \phi_t] = \nabla_\theta \E_\pi[\sum_{t=1}^n \gamma^{t-1} r_t]$.
The RHS is the gradient of the standard $n$-step return, whose maximizer is $\pi^*$.
Since $J_n^\phi$ and the standard objective share the same gradient for all $\theta$, they have the same maximizer.

\paragraph{Approximate case.}
In practice, $V^\phi$ is learned, not zero, so we bootstrap with $\hat{V}^\phi(s_n) \approx V^{\phi,*}(s_n) = 0$.
The bootstrap error $|\hat{V}^\phi(s_n)|$ introduces bias; as training converges ($\hat{V}^\phi \to 0$), the bias vanishes.
This parallels the standard $n$-step bias-variance tradeoff: shorter chunks reduce variance from trajectory noise at the cost of bootstrap bias.
\end{proof}

\subsection{Estimation (Section~\ref{sec:computing})}
\label{app:estimation}

\ThmOptimalProposal*
\begin{proof}
The second moment of the IS estimator is $\E[(\kappa_t(\mathbf{z})/Q(\mathbf{z}))^2 f(\mathbf{z})^2]$.
Since $\Var[f(\mathbf{z})] \leq \sigma^2/(1-\gamma)^2$ (Assumption~\ref{assum:second-moment}), we minimize:
\begin{align}
    \min_Q \sum_{\mathbf{z}} \frac{\kappa_t(\mathbf{z})^2}{Q(\mathbf{z})} \quad \text{subject to } \sum_{\mathbf{z}} Q(\mathbf{z}) = 1.
\end{align}

The $k = 0$ coalition ($\mathbf{z} = \mathbf{0}$, no interventions) has $f(\mathbf{0}) = Y - Y = 0$ deterministically, so it contributes nothing and is excluded from sampling.
Using two-level sampling over $k \in \{1, \ldots, T\}$ (size $k$ with probability $q_k$, then uniform within size $k$): $Q(\mathbf{z}) = q_k / \binom{T}{k}$.
The objective becomes $\sum_{k=1}^{T} c_k/q_k$ where $c_k = \frac{1}{T}(\frac{1}{k} + \frac{1}{T-k})$ for $k < T$ and $c_T = 1/T^2$.
By Cauchy-Schwarz, $q^*_k \propto \sqrt{c_k}$ minimizes this.
Chebyshev gives $(\epsilon,\delta)$-accuracy with $M = O(M_{\delta,\epsilon})$ samples.
\end{proof}

\ThmBootstrappedPhi*
\begin{proof}
The $\lambda$-bootstrapped estimator (Definition~\ref{def:lambda-estimator}) replaces the exact counterfactual return from step $t$ onward with a $\lambda$-return $G^{\mathbf{z},\lambda}_t$ (Eq.~\ref{eq:lambda-return}).
We bound the bias introduced by this substitution.

\textbf{Step 1: $\lambda$-return bias.}
The $\lambda$-return is a mixture of $n$-step returns:
$G^{\mathbf{z},\lambda}_t = (1-\lambda)\sum_{n=1}^{\infty} \lambda^{n-1} G^{(n)}_t$,
where $G^{(n)}_t = \sum_{k=0}^{n-1} \gamma^k y^{\mathbf{z}}_{t+k} + \gamma^n V(s^{\mathbf{z}}_{t+n})$.
Taking conditional expectations given $s^{\mathbf{z}}_t$:
\begin{align}
    \E[G^{(n)}_t \mid s^{\mathbf{z}}_t] = V_\pi(s^{\mathbf{z}}_t) + \gamma^n \E[V(s^{\mathbf{z}}_{t+n}) - V_\pi(s^{\mathbf{z}}_{t+n}) \mid s^{\mathbf{z}}_t].
\end{align}
Since $|V(s) - V_\pi(s)| \leq \epsilon_V$ for all $s$:
\begin{align}
    |\E[G^{\mathbf{z},\lambda}_t \mid s^{\mathbf{z}}_t] - V_\pi(s^{\mathbf{z}}_t)|
    &\leq (1-\lambda) \sum_{n=1}^{\infty} \lambda^{n-1} \gamma^n \epsilon_V
    = \frac{\gamma(1-\lambda)}{1-\gamma\lambda} \epsilon_V. \label{eq:lambda-bias}
\end{align}

\textbf{Step 2: Per-coalition NTE bias.}
The exact NTE is $f(\mathbf{z}) = Y - Y^{\mathbf{z}} = Y - R^{\mathbf{z}}_t - \gamma^{t-1} G^{\mathbf{z},\mathrm{MC}}_t$,
where $G^{\mathbf{z},\mathrm{MC}}_t = \sum_{s \geq t} \gamma^{s-t} y^{\mathbf{z}}_s$ is the Monte Carlo return.
The bootstrapped NTE is $f^{\mathbf{z}}_t = Y - R^{\mathbf{z}}_t - \gamma^{t-1} G^{\mathbf{z},\lambda}_t$.
Their difference is:
\begin{align}
    |E[f^{\mathbf{z}}_t] - f(\mathbf{z})| = \gamma^{t-1} |\E[G^{\mathbf{z},\mathrm{MC}}_t - G^{\mathbf{z},\lambda}_t]| \leq \gamma^{t-1} \cdot \frac{\gamma(1-\lambda)}{1-\gamma\lambda} \epsilon_V \leq \frac{\gamma(1-\lambda)}{1-\gamma\lambda} \epsilon_V,
\end{align}
where the last inequality uses $\gamma^{t-1} \leq 1$ for $t \geq 1$.

\textbf{Step 3: Shapley bias.}
The estimator is $\hat{\phi}^\lambda_t = \sum_{\mathbf{z}} \kappa_t(\mathbf{z}) f^{\mathbf{z}}_t$ (in expectation over the proposal $Q^*$).
The true Shapley value is $\phi_t = \sum_{\mathbf{z}} \kappa_t(\mathbf{z}) f(\mathbf{z})$.
Using the marginal contribution form with non-negative Shapley weights $w(\mathbf{z}) \geq 0$ summing to $1$:
\begin{align}
    |\E[\hat{\phi}^\lambda_t] - \phi_t|
    &= \Bigl|\sum_{\mathbf{z}: z_t = 0} w(\mathbf{z}) \bigl(\E[f^{\mathbf{z} \cup \{t\}}_t - f^{\mathbf{z}}_t] - [f(\mathbf{z} \cup \{t\}) - f(\mathbf{z})]\bigr)\Bigr| \\
    &\leq \sum_{\mathbf{z}: z_t = 0} w(\mathbf{z}) \bigl(|\E[f^{\mathbf{z} \cup \{t\}}_t] - f(\mathbf{z} \cup \{t\})| + |\E[f^{\mathbf{z}}_t] - f(\mathbf{z})|\bigr) \\
    &\leq \frac{2\gamma(1-\lambda)}{1-\gamma\lambda} \epsilon_V. \qedhere
\end{align}
The last line uses Step~2 for each of the two bias terms and $\sum w(\mathbf{z}) = 1$.
Both coalitions $\mathbf{z}$ and $\mathbf{z} \cup \{t\}$ share the same counterfactual states up to step $t$ (the intervention at step $t$ only affects states from $t+1$ onward), so both bootstrap from the same starting state $s^{\mathbf{z}}_t$, and the bound from Step~2 applies to each.
The factor of $2$ arises because the continuation trajectories diverge after step $t$ (different actions), so the two bootstrap biases do not cancel in general.
\end{proof}

\subsection{PTR Buffer Design (Section~\ref{sec:phi-ppo})}
\label{app:ptr}

PTR stores full trajectories with per-timestep priorities.
The priority formula uses a smoothed estimate:
\begin{align}
    p_t = (\phi^{(2)}_t)^{\alpha_{\mathrm{ptr}}/2} + \epsilon_{\mathrm{ptr}},
\end{align}
where $\phi^{(2)}_t$ is a bias-corrected exponential moving average~\citep{kingma2014adam} of squared Shapley estimates, $\alpha_{\mathrm{ptr}}$ controls priority sharpness, and $\epsilon_{\mathrm{ptr}}$ prevents starvation.
Sampling selects (trajectory, starting timestep) pairs proportionally to $p_t$ using the Gumbel-max trick for efficient vectorized sampling.
When the buffer reaches capacity, the trajectory with lowest total priority $\sum_t p_t$ is evicted.

Following standard PER~\citep{schaul2015per}, we apply importance
sampling (IS) correction to account for the non-uniform sampling distribution.
Let $p_i$ denote the priority of sample $i$ in the buffer.
Each sampled subtrajectory receives weight $w_i = (|\mathcal{B}| \cdot p_i / \sum_j p_j)^{-\beta_{\mathrm{is}}}$,
normalized by $\max_i w_i$ for stability. The exponent $\beta_{\mathrm{is}}$ anneals from
$\beta_0$ (e.g., 0.4) to 1 over training, providing full bias correction at
convergence while allowing faster early learning.

\subsection{When Credit Helps (Section~\ref{sec:when-helps})}

\PropSparseCausality*
\begin{proof}
By the Shapley null player axiom (D1), $\phi_t = 0$ whenever $X_t$ has no causal effect on $Y$, so $\hat{g}^\phi = \sum_{t \in \mathcal{C}} g_t$ where $g_t = \phi_t \nabla \log \pi(X_t \mid S_t)$ and $|\mathcal{C}| = k'$.
Expanding the variance of the sum:
\begin{align}
    \Var[\hat{g}^\phi]
    &= \sum_{t \in \mathcal{C}} \Var[g_t] + 2\!\!\sum_{\substack{t, s \in \mathcal{C} \\ t < s}} \Cov[g_t, g_s].
\end{align}
The diagonal terms are bounded by $k' \sigma^2_\phi$.
For the cross-terms, $|\Cov[g_t, g_s]| \leq \bar{\rho}\, \sigma^2_\phi$ by definition of $\bar{\rho}$ and Cauchy-Schwarz on covariance, and there are $\binom{k'}{2}$ pairs.
Combining: $\Var[\hat{g}^\phi] \leq k'\sigma^2_\phi + k'(k'{-}1)\bar{\rho}\,\sigma^2_\phi = k'(1 + (k'{-}1)\bar{\rho})\, \sigma^2_\phi$.

\paragraph{Why $\bar{\rho}$ is small.}
For $t < s$, condition on $(S_s, S_t, X_t)$ and integrate over $X_s \sim \pi(\cdot \mid S_s)$:
\begin{align}
    \Cov[g_t, g_s] = \E\bigl[g_t \cdot \E[\phi_s \nabla\!\log\pi_s \mid S_s, S_t, X_t] \bigr] - \E[g_t]\E[g_s].
\end{align}
The score function identity $\E[\nabla \log \pi(X_s \mid S_s) \mid S_s] = 0$ drives the inner expectation to zero whenever $\phi_s$ is independent of $X_s$ given $S_s$.
Under self-baseline, $\phi_s$ depends on $X_s$ through the on-track condition ($s^{\mathbf{z}}_s = s_s$ when the observed action is reused), introducing a residual correlation.
This dependence is weak because: (i) the on-track condition at step $s$ is primarily determined by interventions at earlier steps, not by $X_s$ itself; and (ii) the optimal proposal $Q^*$ concentrates on extreme coalition sizes where the on-track probability is either very high ($|\mathbf{z}| \approx 0$) or very low ($|\mathbf{z}| \approx T$).
In deterministic environments, $\bar{\rho} = 0$ for on-track coalitions since $\phi_s$ given $S_s$ depends on $X_s$ only through the on-track indicator, which is determined by prior interventions.

In contrast, standard REINFORCE has $\hat{g} = Y \sum_{t=1}^T \nabla \log \pi_t$ where all $T$ terms contribute regardless of causal structure, giving $\Var[\hat{g}] = O(T \sigma^2_Y)$.
The variance ratio is $k'(1 + (k'{-}1)\bar{\rho})/T \approx k'/T$ when $\bar{\rho} \ll 1$.
\end{proof}

\PropStochasticity*
\begin{proof}
The Counterfactual Shapley value is a weighted sum of marginal contributions:
\begin{align}
    \phi_t = \sum_{\mathbf{z}: z_t = 0} w(\mathbf{z}) \cdot [f(\mathbf{z} \cup \{t\}) - f(\mathbf{z})],
\end{align}
where $w(\mathbf{z}) \geq 0$ are Shapley weights summing to 1.
Each marginal contribution compares outcomes $Y^{\mathbf{z} \cup \{t\}}$ and $Y^{\mathbf{z}}$ under shared exogenous noise $\U$.
Let $\sigma^2 = \Var[Y^{\mathbf{z}}]$ (approximately equal for nearby coalitions) and $\rho_Y = \Corr[Y^{\mathbf{z} \cup \{t\}}, Y^{\mathbf{z}}]$.
The variance of a single marginal contribution is:
\begin{align}
    \Var[Y^{\mathbf{z} \cup \{t\}} - Y^{\mathbf{z}}] &= \Var[Y^{\mathbf{z} \cup \{t\}}] + \Var[Y^{\mathbf{z}}] - 2\Cov[Y^{\mathbf{z} \cup \{t\}}, Y^{\mathbf{z}}] \\
    &= 2\sigma^2 - 2\sigma^2 \rho_Y = 2\sigma^2(1 - \rho_Y).
\end{align}
When environment stochasticity dominates action effects, $\rho_Y \approx 1$ (both outcomes are driven by the same exogenous noise), and variance vanishes.
By convexity of variance, $\Var[\phi_t] \leq \max_{\mathbf{z}} \Var[f(\mathbf{z} \cup \{t\}) - f(\mathbf{z})] = O(\sigma^2(1-\rho_Y))$.

In contrast, independent sampling (different $\U$ for each outcome) gives $\rho_Y = 0$ and variance $2\sigma^2$, a factor of $1/(1-\rho_Y)$ larger.
\end{proof}

\PropTDPropagation*
\begin{proof}
Consider tabular TD(0) with terminal-only reward $R_T$ (all $R_t = 0$ for $t < T$).
Initialize $V(s) = 0$ for all states.
The TD(0) update at step $t$ is: $V(S_t) \gets V(S_t) + \alpha[R_t + \gamma V(S_{t+1}) - V(S_t)]$.

\textbf{Episode 1:} Only the transition $(S_{T-1}, S_T)$ has nonzero TD target ($R_T + \gamma \cdot 0 = R_T$).
After episode 1: $V(S_{T-1}) > 0$, all other values remain 0.

\textbf{Episode $n$:} The TD target at step $T-n$ becomes nonzero (since $V(S_{T-n+1}) > 0$ from previous episodes).
Credit propagates backward by one step per episode.

\textbf{After $T$ episodes:} $V(S_1)$ receives its first nonzero update.
This requires $T$ episodes for credit to reach the initial state.

In contrast, $\phi$-redistribution computes $\phi_t$ for all $t$ from a single trajectory via $O(T \cdot M_{\delta,\epsilon})$ counterfactual simulations.
All timesteps receive credit signal in $O(1)$ episodes.
The trade-off: fewer episodes but $O(T \cdot M)$ simulation cost per episode.
\end{proof}

\CorPropagationComplexity*
\begin{proof}
Follows directly from Proposition~\ref{prop:td-propagation}.
For TD(0), each episode updates one state; propagating terminal reward across $T$ states requires $O(T)$ episodes with $O(T)$ updates each, totaling $O(T^2)$.
For $\phi$-redistribution, a single trajectory yields $\phi_t$ for all $T$ timesteps via $O(T \cdot M_{\delta,\epsilon})$ counterfactual simulations, so one episode suffices.
\end{proof}

\section{Additional Experiments}
\label{app:experiments}

This section contains additional experiments validating the theoretical predictions.

\subsection{Fork MDP}
\label{sec:fork-exp}

\begin{figure*}[t]
\centering
\vspace{-0.1in}
\begin{minipage}[c]{0.33\textwidth}
\centering
\input{figures/fork-mdp}
\end{minipage}%
\hfill
\begin{minipage}[c]{0.65\textwidth}
\centering
\includegraphics[width=\textwidth]{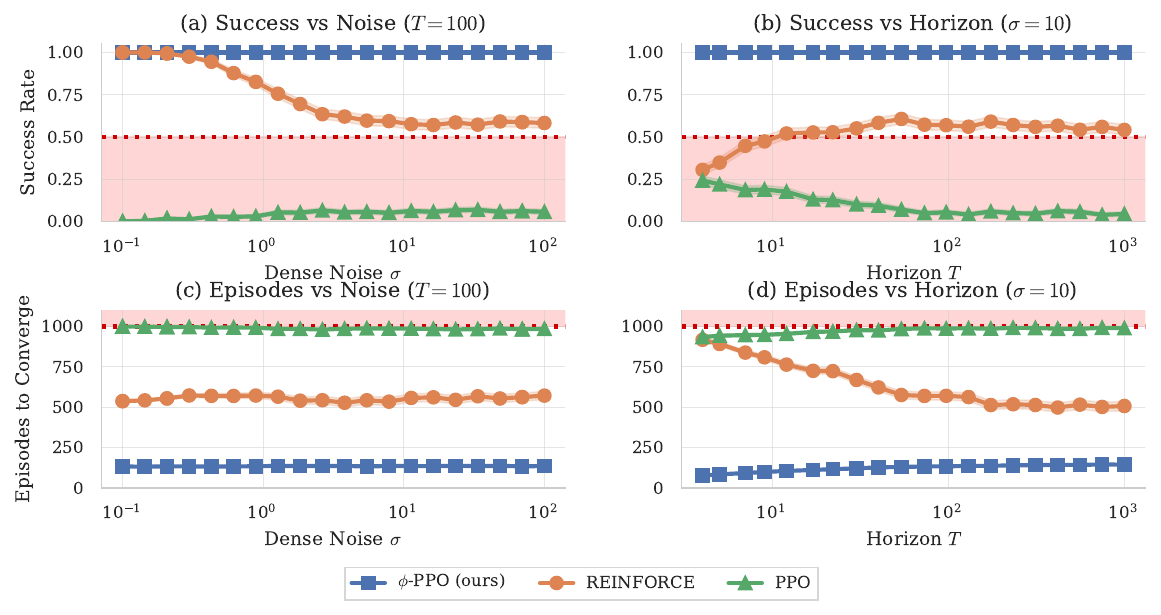}
\end{minipage}
\caption{\textbf{Fork MDP results} ($n{=}1{,}000$ runs per condition; shaded bands are 95\% CIs).
\emph{Left:} MDP structure; only the first $c$ actions affect the terminal reward.
\emph{Top row (a--c):} Success rate across noise $\sigma$, horizon $T$, and causal steps $c$.
\emph{Bottom row (d--f):} Mean episodes to convergence.}
\label{fig:fork}
\end{figure*}

The Fork MDP isolates the three variance reduction benefits of $\phi$-PPO.
The agent chooses from $\mathcal{A} = \{0, 1\}$ at each of $T$ steps ($\gamma{=}0.99$).
Only the first $c$ actions affect the terminal reward: $Y_T \sim \mathcal{N}(1, \sigma^2)$ if $x_1 = \cdots = x_c = 0$ (the optimal actions), $Y_T \sim \mathcal{N}(0, \sigma^2)$ otherwise.
Intermediate rewards are i.i.d.\ noise $Y_t \sim \mathcal{N}(0, \sigma^2)$ for $t < T$.
The environment combines sparse causality ($k' = c \ll T$, Proposition~\ref{prop:sparse-causality}), high stochasticity (shared noise cancels in the counterfactual difference, Proposition~\ref{prop:stochasticity}), and delayed reward ($Y_T$ only, Proposition~\ref{prop:td-propagation}).

\paragraph{Results (Figure~\ref{fig:fork}).}
At the default configuration ($T{=}100$, $\sigma{=}10$, $c{=}1$), $\phi$-PPO achieves $100\%$ success, while REINFORCE degrades to ${\sim}58\%$ and PPO fails at ${\sim}5\%$.
Each column isolates one benefit: noise cancellation (a), horizon independence (b), and causal isolation (c).
The bottom row confirms that $\phi$-PPO also converges fastest: median episodes remain low across all sweeps, while PPO and REINFORCE slow dramatically as conditions worsen.
PPO's clipping saturates under high-variance advantages, preventing meaningful policy updates.
REINFORCE drifts randomly in parameter space, converging roughly half the time.
With $M{=}1$, $\phi$-PPO requires $2T$ steps per episode (one counterfactual rollout), yielding $4{\times}$ episode efficiency but $2{\times}$ simulation efficiency over PPO.

\begin{table}[t]
\centering
\caption{Hyperparameters for $\phi$-PPO (Section~\ref{sec:phi-ppo}) and Fork MDP experiments (Section~\ref{sec:fork-exp}).}
\label{tab:hyperparams}
\begin{tabular}{lll}
\toprule
Parameter & Symbol & Value \\
\midrule
\multicolumn{3}{l}{\emph{$\phi$-PPO (Algo.~\ref{alg:phi-ppo})}} \\
Discount factor & $\gamma$ & 0.99 \\
Trajectory length (chunk size) & $T_{\mathrm{chunk}}$ & 128 \\
Batch size (trajectories) & $B$ & 64 \\
Policy epochs & $K$ & 4 \\
Coalition samples & $M$ & 1 \\
NTE mixing & $\lambda$ & 1.0 (pure MC) \\
GAE mixing & $\lambda^\phi$ & 0.95 \\
PPO clip & $\epsilon_{\mathrm{clip}}$ & 0.2 \\
Value loss coefficient & $c_1$ & 0.5 \\
$\phi$-value loss coefficient & $c_2$ & 0.5 \\
Entropy coefficient & $c_3$ & 0.01 \\
Gradient clip norm & $g_{\max}$ & 0.5 \\
Exploration decay & $\gamma_\xi$ & 0.99 \\
\midrule
\multicolumn{3}{l}{\emph{PTR (Appendix~\ref{app:ptr})}} \\
Priority exponent & $\alpha_{\mathrm{ptr}}$ & 0.6 \\
UCB coefficient & $\beta_{\mathrm{ucb}}$ & 1.0 \\
Priority floor & $\epsilon_{\mathrm{ptr}}$ & 0.01 \\
EMA decay & $\alpha_{\mathrm{ema}}$ & 0.99 \\
IS exponent & $\beta_{\mathrm{is}}$ & $\beta_0 \to 1$ \\
IS exponent start & $\beta_0$ & 0.4 \\
IS annealing iters & $M_\beta$ & 100000 \\
Priority decay & $\gamma_p$ & 0.999 \\
Eviction threshold & $\epsilon_{\mathrm{evict}}$ & $10^{-6}$ \\
\midrule
\multicolumn{3}{l}{\emph{Optimization}} \\
Learning rate & $\eta$ & $0.1$ \\
\midrule
\multicolumn{3}{l}{\emph{Fork MDP environment}} \\
Episode horizon & $T_{\mathrm{ep}}$ & 100 \\
Causal actions & $c$ & 1 \\
Noise std & $\sigma$ & 10 \\
Network architecture & -- & MLP [64, 64] \\
\bottomrule
\end{tabular}
\end{table}

\subsection{Coin Flip MDP: High Stochasticity}

A single-step MDP: $U \sim \text{Bernoulli}(0.5)$, action $X \in \{0,1\}$, return $Y = 100U + X$.
The action's causal effect is $\phi = 1$ with zero variance, while $\Var[Y] = 2500$.
Correlation $\rho = \Corr[Y(0,U), Y(1,U)] \approx 0.9998$.

\subsection{Treasure Hunt MDP: Delayed Reward}

A $T=4$ step MDP with terminal reward only.
Each step: collect treasure or skip.
Reward $Y_3 = \#\{\text{collected}\}$.
All actions are causal ($k' = 4$).
$\phi$-TD provides signal in 1 episode; standard TD requires $O(T) = 4$ episodes for credit propagation.

%% file: figures/fork-mdp.tex

\begin{tikzpicture}[
    state/.style={circle, draw, minimum size=5mm, font=\scriptsize},
    start/.style={state, fill=white},
    mid/.style={state, fill=gray!10, draw=gray!40},
    term/.style={state, fill=yellow!30, draw=orange!50},
    arr/.style={->, thick, >=stealth},
    opt/.style={arr, blue!70},
    sub/.style={arr, red!50},
    lbl/.style={font=\tiny, midway},
    ann/.style={font=\tiny, text=black!60}
]

\node[start] (s0) at (0, 0) {$S_1$};

\node[mid] (s1t) at (1.6, 0.7) {};
\node[ann] at (2.5, 0.7) {$\cdots$};
\node[term] (Yt) at (3.4, 0.7) {$1$};

\node[mid] (s1b) at (1.6, -0.7) {};
\node[ann] at (2.5, -0.7) {$\cdots$};
\node[term] (Yb) at (3.4, -0.7) {$0$};

\draw[opt] (s0) -- node[lbl, above, sloped] {\tiny $X_1{=}0$} (s1t);
\draw[sub] (s0) -- node[lbl, below, sloped] {\tiny $X_1{=}1$} (s1b);
\draw[arr, gray!50] (s1t) -- (2.1, 0.7);
\draw[arr, gray!50] (2.9, 0.7) -- (Yt);
\draw[arr, gray!50] (s1b) -- (2.1, -0.7);
\draw[arr, gray!50] (2.9, -0.7) -- (Yb);

\node[ann, above=1pt of Yt] {$Y_T{\sim}\mathcal{N}(1,\sigma^2)$};
\node[ann, below=1pt of Yb] {$Y_T{\sim}\mathcal{N}(0,\sigma^2)$};
\node[ann] at (1.2, -1.5) {$Y_t \sim \mathcal{N}(0,\sigma^2)$ \scriptsize($t < T$)};

\end{tikzpicture}

%% file: sections/B-limitations.tex
\section{Limitations and Future Work}
\label{app:limitations}
\paragraph{Scope.} Our approach requires a simulator to compute counterfactual trajectories, limiting applicability to environments with resettable state.

\textbf{Limitations.}
Our approach requires simulator access for counterfactual trajectories.
Sample complexity scales with $\sigma^2/(1-\gamma)^2$, which can be large when $\gamma \approx 1$.
We assume counterfactual independence (Assumption~\ref{assum:ctf-indep}), excluding correlated exogenous noise.

\textbf{Future work.}
Finite-sample bounds via Bernstein concentration; relaxing counterfactual independence; empirical validation on Atari/MuJoCo benchmarks.

%% file: sections/B-related-methods.tex
\section{Credit Assignment Methods: Detailed Comparison}
\label{app:ca-methods}

To facilitate the discussion, here we entail the additional notations we use for related work comparison.
$\alpha_{i \to j}$: learned attention weight from token $j$ to token $i$.
$g(\Delta_{0:t})$: LSTM hidden state after processing trajectory up to step $t$.
$P^\pi(X_t {=} x \mid S_t, S_T)$: hindsight posterior over actions given a future state.
$P^\pi(X_t {=} x \mid S_t, Z)$: hindsight posterior over actions given observed return $Z$.
$\phi^\star(s,x,z)$: optimal DICE density-ratio estimator; $\chi^\pi(z \mid s)$: return distribution under $\pi$ from state $s$.
$\Phi_t$: learned hindsight statistic of the future trajectory, constrained to satisfy $X_t \perp \Phi_t \mid S_t$.
$U'$: encoding of a future step $(S', X', Y')$; choices include $U' {=} S'$ (state), $U' {=} Y'$ (reward), or a learned encoding.
$\mathcal{Q}(s,x,\tau)$: quantile function---inverse CDF of the return distribution $F^{-1}_{\eta_{s,x}}(\tau)$ at quantile level $\tau \in [0,1]$.
$\hat\tau$: inferred quantile level of the observed return, $\hat\tau = F_{\eta_{S_t,X_t}}(Z)$.
$B^\pi(s,x,s') = V(s') - \E_{S' \sim p(\cdot|s,x)}[V(S')]$: luck (nature's contribution to the transition).

First we provide an overview of the credit assignment methods we compare against in table~\ref{tab:ca-expressions} with their credit expressions and descriptions.
Then table~\ref{tab:ca-properties} evaluates each method against desirable credit assignment properties (\ref{def:causal credit assignment}) on the Fork MDP
(Section~\ref{sec:fork-exp}: $T{=}100$, $\sigma{=}10$, only $X_1$ causal). Below we discuss the details of each category of credit assignment methods in the literature. 

\paragraph{A. Temporal: discounted reward.}
In the existing credit assignment literature, action values/advantages often serve as a key proxy for actions' influence on the rewards~\citep{watkins1992q,sutton1988temporaldiff,suttonMDPsSemimdpsFramework1999,pmlr-v37-schaul15,10.5555/2031678.2031726generalvaluefunc,pmlr-v139-chang21bmodularity,pan2024skillorluck,auzina-intrinsic-credit-assignment-2026}. 
Although action values and advantages yield unbiased gradients, they assign credit by actions' temporal proximity to the reward, rather than their causal effects. This is both inefficient and biased in terms of credit assignment in that recent actions with no causal effect on the reward add noise to the gradient, while causally relevant actions that occur earlier receive discounted credit.
Intuitively, credit should correspond to actions cause the outcome: we credit studying hard for a good grade, rather than flipping over the returned test to view the grade.

\paragraph{B. Memory: learned models}
Another line of work utilizes learnable sequence models to automatically distribute the credit over actions~\citep{DBLP:conf/nips/ChenLRLGLASM21,arjona2019rudder,ijcai2020selfattentionalcredit,raposo2021synthetic,deng2026densegrpo}, dropping the reliance on the time contiguity assumption. However, such methods conflate agent's actions influence over rewards with environment stochasticity. As a result, they are unable to distinguish skill from luck in learning. Also, learned models suffer from the same problem as black box models that they are also in lack of explainability.

\paragraph{C. Hindsight: posterior/prior ratios and conditional baselines}
The line of hindsight conditioning~\citep{hindsightcredit,velu2024hindsightdice,pmlr-v139-mesnard21a,mesnard2023quantile,meulemans2023would,ramesh2025improving}\footnote{For clarity, prior uses of the term 'counterfactual' in this literature refer to hindsight conditioning not counterfactuals defined in the causal theory~\citep{pearl2009causality}.} partially addresses the skill-luck issue by conditioning on future events and credit each action by how relevant it is to the given event. The fundamental limitation of those is that isolating each single action during attribution overlooks subsequent actions' collective influence on the rewards. Thus, the skill-luck issue still persists as we demonstrate in Figure~\ref{fig:intro example}. 

\begin{table}[ht]
\centering
\caption{Credit assignment methods: expressions and descriptions.}
\label{tab:ca-expressions}
\small
\setlength{\tabcolsep}{3pt}
\begin{tabular}{@{}p{2.8cm}p{3.2cm}p{7.8cm}@{}}
\toprule
\textbf{Method} & \textbf{Expression} & \textbf{Description} \\
\midrule
\multicolumn{3}{l}{\emph{A. Temporal: discounted reward}} \\[2pt]
TD \citep{williams1992simple}, Options \citep{suttonMDPsSemimdpsFramework1999}, UVFA \citep{pmlr-v37-schaul15}, Horde \citep{10.5555/2031678.2031726generalvaluefunc}
  & $\gamma^{T-t} Y_T$
  & Reward $Y_T$ discounted by $\gamma^{T-t}$, the temporal distance from action $X_t$ to reward \\[4pt]
Direct Adv. Estimation \citep{pan2024skillorluck}
  & $\hat{A}(s,x)$
  & Fit $Y = V(S_0) + \sum_t A(S_t, X_t) + \sum_t B^\pi(S_t, X_t, S_{t+1})$ by least squares. $\hat{A}$: agent's causal effect (skill); $B^\pi$: environment's effect (luck) \\[4pt]
QCA \citep{mesnard2023quantile}
  & $\substack{\mathcal{Q}(s,x,\hat\tau) \\{-} \!\sum_{x'}\! \pi(x') \mathcal{Q}(s,x',\hat\tau)}$
  & Quantile level $\hat\tau {=} F_{\eta_{s,x}}(Z)$ indexes ``luck.'' Credit: action's quantile value minus the policy-averaged quantile value at the same luck level \\[4pt]
\midrule
\multicolumn{3}{l}{\emph{B. Memory: learned models}} \\[2pt]
DT \citep{DBLP:conf/nips/ChenLRLGLASM21}
  & $\alpha_{T \to t}$
  & Transformer trained via cross-entropy on actions, conditioned on returns-to-go $\hat{Y}_t {=} \sum_{t' \geq t} Y_{t'}$. Credit implicit via attention weights $\alpha_{T \to t}$ \\[4pt]
RUDDER \citep{arjona2019rudder}
  & $g(\Delta_{0:t}) {-} g(\Delta_{0:t-1})$
  & LSTM $g$ trained via MSE on total return $Y$. Credit to step $t$: change in predicted return after observing $(S_t, X_t)$ \\[4pt]
SECRET \citep{ijcai2020selfattentionalcredit}
  & $\alpha_{t \leftarrow T}\,Y_T$
  & Transformer trained via weighted cross-entropy on $\mathrm{sign}(Y_T)$. Credit: attention $\alpha_{t \leftarrow T}$ on $(S_t, X_t)$, scaled by $Y_T$ \\[4pt]
\midrule
\multicolumn{3}{l}{\emph{C. Hindsight: posterior/prior ratios and conditional baselines}} \\[2pt]
HCA \citep{hindsightcredit}
  & $\dfrac{P^\pi(X_t {=} x \mid S_t, S_T)}{\pi(x \mid S_t)}$
  & Posterior over $X_t$ given future state $S_T$, divided by prior $\pi$. Ratio ${>}\,1$: action became more likely given the future \\[4pt]
H-DICE \citep{velu2024hindsightdice}
  & $\dfrac{\pi(x \mid s)}{P^\pi(X_t {=} x \mid S_t, Z)}$
  & Same hindsight ratio as HCA (conditioned on return $Z$ instead of state). Estimated via DICE variational objective $\phi^\star \!\cdot\! \chi^\pi$ instead of direct density ratio \\[4pt]
CCA \citep{pmlr-v139-mesnard21a}
  & $G_t {-} \E[G_t \!\mid\! S_t, \Phi_t]$
  & $\Phi_t$: learned compression of future trajectory $(S, X, Y)_{t+1:T}$, trained to predict $G_t$ while satisfying $X_t \perp \Phi_t \mid S_t$. Baseline $\E[G_t \mid S_t, \Phi_t]$ fit by regression \\[4pt]
COCOA \citep{meulemans2023would}
  & $\dfrac{P^\pi(U' \mid S_t, X_t)}{P^\pi(U' \mid S_t)} - 1$
  & $U'$: encoding of a future step $(S', X', Y')$; typically $U' {=} Y'$ (reward). Ratio measures how much $X_t$ increased the probability of observing $U'$. Estimated by supervised learning \\[4pt]

\midrule
\multicolumn{3}{l}{\emph{D. Ours: counterfactual simulation}} \\[2pt]
$\phi$-\textbf{values}
  & $\mathbf{Y - Y^{z}}$, \textbf{shared} $\U$
  & Replay trajectory in simulator with alternative action $X'_t {\sim} \pi$ and same exogenous noise $\U$. Credit: difference in outcomes $Y - Y^{\mathbf{z}}$ \\
\bottomrule
\end{tabular}
\end{table}

\begin{table}[ht]
\centering
\caption{Desirable credit assignment properties, following the axiomatic framework of \citet{lee2025}.
\textbf{D1}~(Admissibility): non-causal actions receive zero credit.
\textbf{D2}~(Power): causal actions receive nonzero credit.
\textbf{D3}~(Normality): given equal causal effects, actions whose counterfactual baselines are more probable under $\pi$ receive higher credit (actions $\pi$ would have taken anyway get less credit).
\textbf{D4}~(Effect scaling): given equal baselines, actions with larger causal effects receive proportionally more credit.
\textbf{D5}~(Efficiency): credit sum to total return, $\sum_t \phi_t = Y - Y^{\mathbf{0}}$.
\checkmark: satisfied. $\boldsymbol{\sim}$: partially or in principle but not in practice. \texttimes: violated.}
\label{tab:ca-properties}
\small
\setlength{\tabcolsep}{4pt}
\begin{tabular}{@{}lcccccc@{}}
\toprule
\textbf{Method} & \textbf{D1} & \textbf{D2} & \textbf{D3} & \textbf{D4} & \textbf{D5} \\
 & Admiss. & Power & Normal. & Scaling & Effic. \\
\midrule
\multicolumn{6}{l}{\emph{A. Temporal}} \\
TD         & \texttimes & $\boldsymbol{\sim}$ & \texttimes & $\boldsymbol{\sim}$ & \texttimes \\
Options        & \texttimes & $\boldsymbol{\sim}$ & \texttimes & $\boldsymbol{\sim}$ & \texttimes \\
UVFA           & \texttimes & $\boldsymbol{\sim}$ & \texttimes & $\boldsymbol{\sim}$ & \texttimes \\
Horde          & \texttimes & $\boldsymbol{\sim}$ & \texttimes & $\boldsymbol{\sim}$ & \texttimes \\
Direct Advantage Estimation     & $\boldsymbol{\sim}$ & \texttimes & \texttimes & \texttimes & \texttimes \\
QCA (exact)    & \checkmark & \checkmark & \texttimes & \checkmark & \texttimes \\
QCA (learned)  & \checkmark & \texttimes & \texttimes & \texttimes & \texttimes \\
\midrule
\multicolumn{6}{l}{\emph{B. Memory}} \\
DT             & \texttimes & \texttimes & \texttimes & \texttimes & \texttimes \\
RUDDER         & \texttimes & $\boldsymbol{\sim}$ & \texttimes & \texttimes & \checkmark \\
SECRET         & \texttimes & \texttimes & \texttimes & \texttimes & $\boldsymbol{\sim}$ \\
\midrule
\multicolumn{6}{l}{\emph{C. Hindsight}} \\
HCA            & $\boldsymbol{\sim}$ & \texttimes & $\boldsymbol{\sim}$ & \texttimes & \texttimes \\
H-DICE         & $\boldsymbol{\sim}$ & \texttimes & $\boldsymbol{\sim}$ & \texttimes & \texttimes \\
CCA            & $\boldsymbol{\sim}$ & $\boldsymbol{\sim}$ & \texttimes & $\boldsymbol{\sim}$ & \texttimes \\
COCOA          & \checkmark & \texttimes & $\boldsymbol{\sim}$ & \texttimes & \texttimes \\
\midrule
\multicolumn{6}{l}{\emph{D. Ours}} \\
$\phi$-\textbf{values} & \checkmark & \checkmark & \checkmark & \checkmark & \checkmark \\
\bottomrule
\end{tabular}
\end{table}